\documentclass{article}

\usepackage{iclr2019_conference,times}
\iclrfinalcopy

\usepackage{amsmath,amsfonts,bm}

\def\eqref#1{equation~\ref{#1}}

\def\1{\bm{1}}

\def\rr{{\textnormal{r}}}

\def\vv{{\bm{v}}}

\DeclareMathAlphabet{\mathsfit}{\encodingdefault}{\sfdefault}{m}{sl}
\SetMathAlphabet{\mathsfit}{bold}{\encodingdefault}{\sfdefault}{bx}{n}

\DeclareMathOperator*{\argmax}{arg\,max}
\DeclareMathOperator*{\argmin}{arg\,min}

\usepackage[utf8]{inputenc} 
\usepackage[T1]{fontenc}    
\usepackage{hyperref}       
\usepackage{url}            
\usepackage{booktabs}       
\usepackage{amsfonts}       
\usepackage{nicefrac}       
\usepackage{microtype}      

\usepackage{amsmath}
\usepackage{amsfonts}
\usepackage{amssymb}
\usepackage{mleftright}
\usepackage{xparse}

\usepackage{float}
\usepackage{algorithm}
\usepackage{algpseudocode}
\usepackage{graphicx}
\usepackage{subcaption}
\usepackage{dirtytalk}
\usepackage{multirow}
\usepackage{bbm}
\usepackage{bm}

\usepackage{url}
\usepackage{booktabs}
\usepackage{amsfonts}
\usepackage{nicefrac}
\usepackage{microtype}
\usepackage{natbib}
\usepackage{apalike}
\usepackage{dsfont}
\usepackage{xcolor}

\setcitestyle{citesep={,}}  

\usepackage{stmaryrd}
\usepackage{amsthm}

\newtheorem*{rep@theorem}{\rep@title}
\newcommand{\newreptheorem}[2]{%
\newenvironment{rep#1}[1]{%
 \def\rep@title{#2 \ref{##1}}%
 \begin{rep@theorem}}%
 {\end{rep@theorem}}}
\makeatother

\newreptheorem{theorem}{Theorem}
\newtheorem{lemma}{Lemma}
\newreptheorem{lemma}{Lemma}

\newreptheorem{corollary}{Corollary}
\newtheorem{proposition}{Proposition}
\newreptheorem{proposition}{Proposition}

\newreptheorem{observation}{Observation}

\newtheorem{innercustomthm}{Proposition}
\newenvironment{customthm}[1]
  {\renewcommand\theinnercustomthm{#1}\innercustomthm}
  {\endinnercustomthm}

\NewDocumentCommand{\evalat}{sO{\big}mm}{%
  \IfBooleanTF{#1}
   {\mleft. #3 \mright|_{#4}}
   {#3#2|_{#4}}%
}

\newcommand{\w}{\mathbf{w}}
\newcommand{\x}{\mathbf{x}}
\newcommand{\s}{\mathbf{s}}
\newcommand{\f}{\mathbf{f}}

\newcommand{\uu}{\mathbf{u}}
\newcommand{\z}{\mathbf{z}}
\newcommand{\T}{\mathcal{T}}
\newcommand{\Y}{\mathcal{Y}}
\newcommand{\LL}{\mathcal{L}}

\newcommand{\ybar}{\bar{y}}
\newcommand{\what}{\hat{\w}}
\newcommand{\aalpha}{\bm{\alpha}}
\newcommand{\ddelta}{\bm{\delta}}
\newcommand{\bb}{\mathbf{b}}
\newcommand{\PP}{\mathcal{P}}

\newsavebox{\leftbox}
\newsavebox{\rightbox}

\title{Deep Frank-Wolfe \\ For Neural Network Optimization}

\author{Leonard Berrada$^1$, Andrew Zisserman$^1$ and M. Pawan Kumar$^{1, 2}$ \\
  $^1$Department of Engineering Science\\
 \hspace{3pt} University of Oxford\\
  $^2$Alan Turing Institute\\
  \texttt{\{lberrada,az,pawan\}@robots.ox.ac.uk} \\
}

\begin{document}

\maketitle

\begin{abstract}
Learning a deep neural network requires solving a challenging optimization problem: it is a high-dimensional, non-convex and non-smooth minimization problem with a large number of terms.
The current practice in neural network optimization is to rely on the stochastic gradient descent (SGD) algorithm or its adaptive variants.
However, SGD requires a hand-designed schedule for the learning rate.
In addition, its adaptive variants tend to produce solutions that generalize less well on unseen data than SGD with a hand-designed schedule.
We present an optimization method that offers empirically the best of both worlds: our algorithm yields good generalization performance while requiring only one hyper-parameter.
Our approach is based on a composite proximal framework, which exploits the compositional nature of deep neural networks and can leverage powerful convex optimization algorithms by design.
Specifically, we employ the Frank-Wolfe (FW) algorithm for SVM, which computes an optimal step-size in closed-form at each time-step.
We further show that the descent direction is given by a simple backward pass in the network, yielding the same computational cost per iteration as SGD.
We present experiments on the CIFAR and SNLI data sets, where we demonstrate the significant superiority of our method over Adam, Adagrad, as well as the recently proposed BPGrad and AMSGrad.
Furthermore, we compare our algorithm to SGD with a hand-designed learning rate schedule, and show that it provides similar generalization while often converging faster.
The code is publicly available at \url{https://github.com/oval-group/dfw}.
\end{abstract}

\section{Introduction}

Since the introduction of back-propagation \citep{Rumelhart1986}, stochastic gradient descent (SGD) has been the most commonly used optimization algorithm for deep neural networks.
While yielding remarkable performance on a variety of learning tasks, a downside of the SGD algorithm is that it requires a schedule for the decay of its learning rate.
In the convex setting, curvature properties of the objective function can be used to design schedules that are hyper-parameter free and guaranteed to converge to the optimal solution \citep{Bubeck2015}.
However, there is no analogous result of practical interest for the non-convex optimization problem of a deep neural network.
An illustration of this issue is the diversity of learning rate schedules used to train deep convolutional networks with SGD: \cite{Simonyan2015} and \cite{He2016} adapt the learning rate according to the validation performance, while \cite{Szegedy2015,Huang2017a} and \cite{Loshchilov2017} use pre-determined schedules, which are respectively piecewise constant, geometrically decaying, and cyclic with a cosine annealing.
While these protocols result in competitive or state-of-the-art results on their learning task, there does not seem to be a consistent methodology.
As a result, finding such a schedule for a new setting is a time-consuming and computationally expensive effort.

To alleviate this issue, adaptive gradient methods have been developed \citep{Zeiler2012,Kingma2015,Reddi2018}, and borrowed from online convex optimization \citep{Duchi2011}.
Typically, these methods only require the tuning of the initial learning rate, the other hyper-parameters being considered robust across applications.
However, it has been shown that such adaptive gradient methods obtain worse generalization than SGD \citep{Wilson2017}.
This observation is corroborated by our experimental results.

In order to bridge this performance gap between existing adaptive methods and SGD, we introduce a new optimization algorithm, called Deep Frank-Wolfe (DFW). The DFW algorithm exploits the composite structure of deep neural networks to design an optimization algorithm that leverages efficient convex solvers.
In more detail, we consider a composite (nested) optimization problem, with the loss as the outer function and the function encoded by the neural network as the inner one.
At each iteration, we define a proximal problem with a first-order approximation of the neural network (linearized inner function), while keeping the loss function in its exact form (exact outer function).
When the loss is the hinge loss, each proximal problem created by our formulation is exactly a linear SVM.
This allows us to employ the powerful Frank-Wolfe (FW) algorithm as the workhorse of our procedure.

There are two by-design advantages to our method compared to the SGD algorithm.
First, each iteration exploits more information about the learning objective, while preserving the same computational cost as SGD.
Second, an optimal step-size is computed in closed-form by using the FW algorithm in the dual \citep{Frank1956,Lacoste-Julien2013}.
Consequently, we do not need a hand-designed schedule for the learning rate.
As a result, our algorithm is the first to provide competitive generalization error compared to SGD, all the while requiring a single hyper-parameter and often converging significantly faster.

We present two additional improvements to customize the use of the DFW algorithm to deep neural networks.
First, we show how to smooth the loss function to avoid optimization difficulties arising from learning deep models with SVMs \citep{Berrada2018}.
Second, we incorporate Nesterov momentum \citep{Nesterov1983} to accelerate our algorithm.

We demonstrate the efficacy of our method on image classification with the CIFAR data sets \citep{Krizhevsky2009} using two architectures: wide residual networks \citep{Zagoruyko2016} and densely connected convolutional neural networks \citep{Huang2017a}; we also provide experiments on natural language inference with a Bi-LSTM on the SNLI corpus \citep{Bowman2015}.
We show that the DFW algorithm often strongly outperforms previous methods based on adaptive learning rates.
Furthermore, it provides comparable or better accuracy to SGD with hand-designed learning rate schedules.

In conclusion, our contributions can be summed up as follows:
\begin{itemize}
    \item We propose a proximal framework which preserves information from the loss function.
    \item For the first time for deep neural networks, we demonstrate how our formulation gives at each iteration (i) an optimal step-size in closed form and (ii) an update at the same computational cost as SGD.
    \item We design a novel smoothing scheme for the dual optimization of SVMs.
    \item To the best of our knowledge, the resulting DFW algorithm is the first to offer comparable or better generalization to SGD with a hand-designed schedule on the CIFAR data sets, all the while converging several times faster and requiring only a single hyperparameter.
\end{itemize}

\section{Related Work}

\paragraph{Non Gradient-Based Methods.} The success of a simple first-order method such as SGD has led to research in other more sophisticated techniques based on relaxations~\citep{Heinemann2016,Zhang2017}, learning theory~\citep{Goel2017}, Bregman iterations~\citep{Taylor2016}, and even second-order methods~\citep{Roux2008,Martens2012,Ollivier2013,Desjardins2015,Martens2015,Grosse2016,Ba2017,Botev2017,Martens2018}.
While such methods hold a lot of promise, their relatively large per-iteration cost limits their scalability in practice.
As a result, gradient-based methods continue to be the most popular optimization algorithms for learning deep neural networks.

\paragraph{Adaptive Gradient Methods.} As mentioned earlier, one of the main challenges of using SGD is the design of a learning rate schedule.
Several works proposed alternative first-order methods that do not require such a schedule, by either modifying the descent direction or adaptively rescaling the step-size \citep{Duchi2011,Zeiler2012,Schaul2013,Kingma2015,Zhang2017a,Reddi2018}.
However, as noted above, the adaptive variants of SGD sometimes provide subpar generalization~\citep{Wilson2017}.

\paragraph{Learning to Learn and Meta-Learning.} Learning to learn approaches have also been proposed to optimize deep neural networks.
\cite{Baydin2018} and \cite{Wu2018} learn the learning rate to avoid a hand-designed schedule and to improve practical performance.
Such methods can be combined with our proposed algorithm to learn its proximal coefficient, instead of considering it as a fixed hyper-parameter to be tuned.
Meta-learning approaches have also been suggested to learn the optimization algorithm \citep{Andrychowicz2016,Ravi2017,Wichrowska2017,Li2017b}.
This line of work, which is orthogonal to ours, could benefit from the use of DFW to optimize the meta-learner.

\paragraph{Optimization and Generalization.} Several works study the relationship between optimization and generalization in deep learning.
In order to promote generalization within the optimization algorithm itself, \cite{Neyshabur2015,Neyshabur2016} proposed the Path-SGD algorithm, which implicitly controls the capacity of the model.
However, their method required the model to employ ReLU non-linearity only, which is an important restriction for practical purposes.
\cite{Hardt2016,Arpit2017,Neyshabur2017,Hoffer2017} and \cite{Chaudhari2018} analyzed how existing optimization algorithms implicitly regularize deep neural networks.
However this phenomenon is not yet fully understood, and the resulting empirical recommendations are sometimes opposing \citep{Hardt2016,Hoffer2017}.

\paragraph{Proximal Methods.} The back-propagation algorithm has been analyzed in a proximal framework in \citep{Frerix2018}.
Yet, the resulting approach still requires the same hyper-parameters as SGD and incurs a higher computational cost per iteration.

\paragraph{Linear SVM Sub-Problems.} A main component of our formulation is to formulate sub-problems as linear SVMs.
In an earlier work \citep{Berrada2017}, we showed that neural networks with piecewise linear activations could be trained with the CCCP algorithm \citep{Yuille2002}, which yielded approximate SVM problems to be solved with the BCFW algorithm \citep{Lacoste-Julien2013}.
However this algorithm only updates the parameters of one layer at a time, which slows down convergence significantly in practice.
Closest to our approach are the works of \citep{Hochreiter2005} and \citep{Singh2018}.
\cite{Hochreiter2005} suggested to create a local SVM based on a first-order Taylor expansion and a proximal term, in order to lower the error of every data sample while minimizing the changes in the weights.
However their method operated in a non-stochastic setting, making the approach infeasible for large-scale data sets.
\cite{Singh2018}, a parallel work to ours, also created an SVM problem using a first-order Taylor expansion, this time in a mini-batch setting.
Their work provided interesting insights from a statistical learning theory perspective.
While their method is well-grounded, its significantly higher cost per iteration impairs its practical speed and scalability.
As such, it can be seen as complementary to our empirical work, which exploits a powerful solver and provides state-of-the-art scalability and performance.

\section{Problem Formulation}
\label{sec:problem_formulation}

Before describing our formulation, we introduce some necessary notation.
We use $\| \cdot \|$ to denote the Euclidean norm.
Given a function $\phi$, $\evalat{\partial\phi(\uu)}{\hat{\uu}}$ is the derivative of $\phi$ with respect to $\uu$ evaluated at $\hat{\uu}$.
According to the situation, this derivative can be a gradient, a Jacobian or even a directional derivative.
Its exact nature will be clear from context throughout the paper.
We also introduce the first-order Taylor expansion of $\phi$ around the point $\hat{\uu}$: $\T_{\hat{\uu}}\phi(\uu) = \phi(\hat{\uu}) + (\evalat{\partial\phi(\uu)}{\hat{\uu}})^\top(\uu - \hat{\uu})$.
For a positive integer $p$, we denote the set $\{1, 2, ..., p\}$ as $[p]$.
For simplicity, we assume that stochastic algorithms process only one sample at each iteration, although the methods can be trivially extended to mini-batches of size larger than one.

\subsection{Learning Objective}

We suppose we are given a data set $(\x_i, y_i)_{i \in [N]}$, where each $\x_i \in \mathbb{R}^d$ is a sample annotated with a label $y_i$ from the output space $\Y$.
The data set is used to estimate a parameterized model represented by the function $\f$.
Given its (flattened) parameters $\w \in \mathbb{R}^p$, and an input $\x_i \in \mathbb{R}^d$, the model predicts $\f(\w, \x_i) \in \mathbb{R}^{|\Y|}$, a vector with one score per element of the output space $\Y$.
For instance, $\f$ can be a linear map or a deep neural network.
Given a vector of scores per label $\s \in \mathbb{R}^{|\Y|}$, we denote by $\LL(\s, y_i)$ the loss function that computes the risk of the prediction scores $\s$ given the ground truth label $y_i$. For example, the loss $\LL$ can be cross-entropy or the multi-class hinge loss:
\begin{equation}
\text{(Cross-Entropy Loss)} \quad \LL_{CE}: (\s, y) \in \mathbb{R}^{|\Y|} \times \Y \mapsto \log \left(\sum_{k \in \Y} \exp (s_{k}) \right) - s_y, \label{eq:ce_loss}
\end{equation}
\begin{equation}
\text{(Multi-Class Hinge Loss)} \quad \LL_{hinge}: (\s, y) \in \mathbb{R}^{|\Y|} \times \Y \mapsto \max \left\{ \max_{k \in \Y \backslash \{y\}} \left\{ s_{k} + 1 - s_y \right\}, 0 \right\}. \label{eq:hinge_loss}
\end{equation}
The cross-entropy loss (\ref{eq:ce_loss}) tries to match the empirical distribution by driving incorrect scores as far as possible from the ground truth one.
The hinge loss (\ref{eq:hinge_loss}) attempts to create a minimal margin of one between correct and incorrect scores. The hinge loss has been shown to be more robust to over-fitting than cross-entropy, when combined with smoothing techniques that are common in the optimization literature \citep{Berrada2018}.
To simplify notation, we introduce $\f_i(\w) = \f(\w, \x_i)$ and $\LL_i(\s) = \LL(\s, y_i)$ for each $i \in [N]$.
Finally, we denote by $\rho(\w)$ the regularization (typically the squared Euclidean norm).
We now write the learning problem under its empirical risk minimization form:
\begin{equation} \label{eq:svm_primal}
\min\limits_{\w \in \mathbb{R}^p} \rho(\w) + \dfrac{1}{N}\sum\limits_{i \in [N]} \LL_i (\f_i(\w)).
\end{equation}

\subsection{A Proximal Approach}

Our main contribution is a formulation which exploits the composite nature of deep neural networks in order to obtain a better approximation of the objective at each iteration.
Thanks to the careful approximation design, this approach yields sub-problems that are amenable to efficient optimization by powerful convex solvers.
 In order to understand the intuition of our approach, we first present a proximal gradient perspective on SGD.

\paragraph{The SGD Algorithm.} At iteration $t$, the SGD algorithm selects a sample $j$ at random and observes the objective estimate $\rho(\w_t) + \LL_j ( \f_j(\w_t))$.
Then, given the learning rate $\eta_t$, it performs the following update on the parameters:
\begin{equation} \label{eq:sgd_step}
\w_{t+1} = \w_t - \eta_t \left( \evalat{\partial\rho(\w)}{\w_t} + \evalat{\partial \LL_j (\f_j(\w))}{\w_t}\right).
\end{equation}
Equation (\ref{eq:sgd_step}) is the closed-form solution of a proximal problem where the objective has been linearized by the first-order Taylor expansion $\T_{\w_t}$ \citep{Bubeck2015}:
\begin{equation} \label{eq:sgd_step_min}
\w_{t+1} = \argmin\limits_{\w \in \mathbb{R}^p} \left\{ \dfrac{1}{2 \eta_t} \|\w - \w_t \|^2 + \T_{\w_t}\rho(\w) + \T_{\w_t} [\LL_j (\f_j(\w))] \right\}.
\end{equation}
To see the relationship between (\ref{eq:sgd_step}) and (\ref{eq:sgd_step_min}), one can set the gradient with respect to $\w$ to 0 in equation (\ref{eq:sgd_step_min}), and observe that the resulting equation is exactly (\ref{eq:sgd_step}).
In other words, SGD minimizes a first-order approximation of the objective, while encouraging proximity to the current estimate $\w_t$.

However, one can also choose to linearize only a part of the composite objective \citep{Lewis2016}.
Choosing which part to approximate is a crucial decision, because it yields optimization problems with widely different properties.
In this work, we suggest an approach that lends itself to fast optimization with robust convex solvers and preserves information about the learning task by keeping an exact loss function.

\paragraph{Loss-Preserving Linearization.} In detail, at iteration $t$, with selected sample $j$, we introduce the proximal problem that linearizes the regularization $\rho$ and the model $\f_j$, but not the loss function $\LL$:
\begin{equation} \label{eq:proxfw_step_min}
\min\limits_{\w \in \mathbb{R}^p} \left\{ \dfrac{1}{2 \eta_t} \|\w - \w_t \|^2 + \T_{\w_t}\rho(\w) + \LL_j (\T_{\w_t}\f_j(\w)) \right\}.
\end{equation}

\begin{figure}[h]
\centering
\begin{minipage}{.47\textwidth}
  \centering
\includegraphics[height=0.65\linewidth]{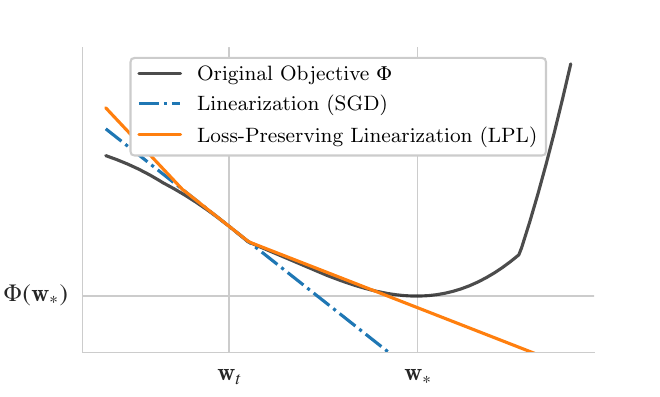}
\end{minipage}%
\hfill
\begin{minipage}{.47\textwidth}
  \centering
  \includegraphics[height=0.65\linewidth]{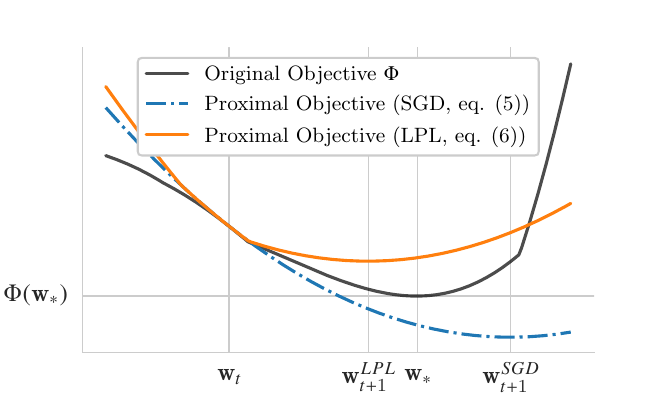}
\end{minipage}
\caption{\em We illustrate the different approximations on a synthetic composite objective function $\Phi(\w) = \LL(\f(\w))$ ($\Phi$ is plotted in black).
In this example, $\LL$ is a maximum of linear functions (similarly to a hinge loss) and $\f$ is a non-linear smooth map.
We denote the current iterate by $\w_t$, and the point minimizing $\Phi$ by $\w_*$.
On the left-hand side, one can observe how the SGD approximation is a single line (tangent at $\Phi(\w_t)$, in blue), while the LPL approximation is piecewise linear (in orange), and thus matches the objective curve (in black) more closely.
On the right-hand side, an identical proximal term is added to both approximations to visualize equations (\ref{eq:sgd_step_min}) and (\ref{eq:proxfw_step_min}).
Thanks to the better accuracy of the LPL approximation, the iterate $\w_{t+1}^\text{LPL}$ gets closer to the solution $\w_*$ than $\w_{t+1}^\text{SGD}$.
This effect is particularly true when the proximal coefficient $\frac{1}{2 \eta_t}$ is small, or equivalently, when the learning rate $\eta_t$ is large.
Indeed, the accuracy of the local approximation becomes more important when the proximal term is contributing less (e.g. when $\eta_t$ is large).
}
\label{fig:approximation}
\end{figure}

In figure \ref{fig:approximation}, we provide a visual comparison of equations (\ref{eq:sgd_step_min}) and (\ref{eq:proxfw_step_min}) in the case of a piecewise linear loss.
As will be seen, by preserving the loss function, we will be able to achieve good performance across a number of tasks with a fixed $\eta_t=\eta$.
Consequently, we will provide the first algorithm to accurately learn deep neural networks with only a single hyper-parameter while offering similar performance compared to SGD with a hand-designed schedule.

\section{The Deep Frank-Wolfe Algorithm}

\subsection{Algorithm}

We focus on the optimization of equation (\ref{eq:proxfw_step_min}) when $\LL$ is a multi-class hinge loss (\ref{eq:hinge_loss}).
The results of this section were originally derived for linear models \citep{Lacoste-Julien2013}.
Our contribution is to show for the first time how they can be exploited for deep neural networks thanks to our formulation (\ref{eq:proxfw_step_min}).
We will refer to the resulting algorithm for neural networks as Deep Frank-Wolfe (DFW).
We begin by stating the key advantage of our method.
\begin{proposition}[Optimal step-size, \citep{Lacoste-Julien2013}] \label{prop:opt_stepsize}
Problem (\ref{eq:proxfw_step_min}) with a hinge loss is amenable to optimization with Frank-Wolfe in the dual, which yields an optimal step-size $\gamma_t \in [0, 1]$ in closed-form at each iteration $t$.
\end{proposition}

This optimal step-size can be obtained in closed-form because the hinge loss is convex and piecewise linear.
In fact, the approach presented here can be applied to any loss function  $\LL$ that is convex and piecewise linear (another example would be the $l_1$ distance for regression for instance).

Since the step-size can be computed in closed-form, the main computational challenge is to obtain the update direction, that is, the conditional gradient of the dual.
In the following result, we show that by taking a single step per proximal problem, this dual conditional gradient can be computed at the same cost as a standard stochastic gradient.
The proof is available in appendix \ref{sec:appendix_simplified_fw_algo}.

\begin{proposition}[Cost per iteration] \label{prop:cond_gradient_primal}
If a single step is performed on the dual of (\ref{eq:proxfw_step_min}), its conditional gradient is given by $-\evalat{\partial \left(\rho(\w) + \LL_{y}(\f_\x(\w)) \right)}{\w_t}$.
Given the step-size $\gamma_t$, the resulting update can be written as:
\begin{equation} \label{eq:primal_update}
\w_{t+1} = \w_t - \eta \left[\evalat{\partial \rho(\w)}{\w_t} + \gamma_t \evalat{\partial \LL_j(\f_j(\w))}{\w_t} \right]
\end{equation}
\end{proposition}
In other words, the cost per iteration of the DFW algorithm is the same as SGD, since the update only requires standard stochastic gradients.
In addition, we point out that in a mini-batch setting, the conditional gradient is given by the average of the gradients over the mini-batch.
As a consequence, we can use batch Frank-Wolfe in the dual rather than coordinate-wise updates, with the same parallelism as SGD over the samples of a mini-batch.

As we detail in appendix \ref{sec:appendix_simplified_fw_algo}, the direction given in Proposition \ref{prop:cond_gradient_primal} is actually an inexact but close approximation to the conditional gradient when $\rho \neq 0$, and it does exactly match the conditional gradient when $\rho = 0$. In all cases, it corresponds to a feasible direction in the dual. For simplicity purposes, we still refer to this direction as the conditional gradient.

One can observe how the update (\ref{eq:primal_update}) exploits the optimal step-size $\gamma_t \in [0, 1]$ given by Proposition \ref{prop:opt_stepsize}.
There is a geometric interpretation to the role of this step-size $\gamma_t$.
When $\gamma_t$ is set to its minimal value 0, the resulting iterate does not move along the direction $\evalat{\partial \LL_j(\f_j(\w))}{\w_t}$.
Since the step-size is optimal, this can only happen if the current iterate is detected to be at a minimum of the piecewise linear approximation.
Conversely, when $\gamma_t$ reaches its maximal value 1, the algorithm tries to move as far as possible along the direction $\evalat{\partial \LL_j(\f_j(\w))}{\w_t}$.
In that case, the update is the same as the one obtained by SGD (as given by equation (\ref{eq:sgd_step})).
In other words, $\gamma_t$ can automatically decay the effective learning rate, hereby preventing the need to design a learning rate schedule by hand.

As mentioned previously, the DFW algorithm performs only one step per proximal problem.
Since problem (\ref{eq:proxfw_step_min}) is only an approximation of the original problem (\ref{eq:svm_primal}), it may be unnecessarily expensive to solve it very accurately.
Therefore taking a single step per proximal problem may help the DFW algorithm to converge faster.
This is confirmed by our experimental results, which show that DFW is often able to minimize the learning objective (\ref{eq:svm_primal}) at greater speed than SGD.

\subsection{Improvements for Deep Neural Networks}

We present two improvements to customize the application of our algorithm to deep neural networks.

\paragraph{Smoothing.}
The SVM loss is non-smooth and has sparse derivatives, which can cause difficulties when training a deep neural network \citep{Berrada2018}.
In Appendix \ref{sec:appendix_smoothing}, we derive a novel result that shows how we can exploit the smooth primal cross-entropy direction and inexpensively detect when to switch back to using the standard conditional gradient.

\paragraph{Nesterov Momentum.}
To take advantage of acceleration similarly to the SGD baseline, we adapt the Nesterov momentum to the DFW algorithm.
We defer the details to the appendix in \ref{sec:appendix_nesterov} for space reasons.
We further note that the momentum coefficient $\mu$ is typically set to a high value, say 0.9, and does not contribute significantly to the computational cost of cross-validation.

\subsection{Algorithm Summary}

The main steps of DFW are shown in Algorithm \ref{algo:dfw}. As the key feature of our approach, note that the step-size is computed in closed-form in step \ref{algo_line:step_size} of the algorithm (colored in blue).

\begin{algorithm}[h]
\caption{\em The Deep Frank-Wolfe Algorithm}\label{algo:dfw}
\begin{algorithmic}[1]
\Require proximal coefficient $\eta$, initial point $\w_0 \in \mathbb{R}^p$, momentum coefficient $\mu$, number of epochs
\State $t=0$
\State $\z_0=0$ \Comment{Momentum velocity (initialization)}
\For {each epoch}
\For {each mini-batch $\mathcal{B}$}
    \State Receive data of mini-batch $(\x_i, y_i)_{i \in \mathcal{B}}$
    \State $\forall i \in \mathcal{B}, \: \bb^{(i)}_t(\w_t) = (f_{\x_i, \ybar}(\w_t) - f_{\x_i, y_i}(\w_t) + \Delta(\ybar, y_i))_{\ybar \in \Y}$ \Comment{Forward pass}
    \State $\forall i \in \mathcal{B}, \: \s_t^{(i)} \gets \texttt{get\_s}(\bb^{(i)}_t(\w_t))$ \Comment{Dual direction (details in Appendix \ref{sec:appendix_smoothing})} \label{algo_line:dual_search}
    \State $\ddelta_t = \partial \left( \evalat{\frac{1}{|\mathcal{B}|} \sum_{i \in \mathcal{B}} (\s_t^{(i)})^\top \bb_t^{(i)}(\w) \right)}{\w_t}$ \Comment{Derivative of (smoothed) loss function}
    \State $\rr_t = \evalat{\partial \rho(\w)}{\w_t}$ \Comment{Derivative of regularization}
    \State $\color{blue} \gamma_t = (-\eta \ddelta_t^\top \rr_t  + \frac{1}{|\mathcal{B}|} \sum_{i \in \mathcal{B}} (\s_t^{(i)})^\top \bb_t^{(i)}(\w_t) / (\eta\| \ddelta_t\|^2)$ {\color{blue} clipped to [0, 1]} \Comment{Step-size} \label{algo_line:step_size}
    \State $\z_{t+1} =  \mu \z_t - \eta \gamma_t (\rr_t + \ddelta_t)$ \Comment{Velocity accumulation}
    \State $\w_{t+1} = \w_t - \eta \left[\rr_t + \gamma_t \ddelta_t \right] + \mu \z_{t+1}$ \Comment{Parameters update} \label{algo_line:update}
    \State $t=t+1$
    \EndFor
\EndFor
\end{algorithmic}
\end{algorithm}

Note that only the hyper-parameter $\eta$ will be tuned in our experiments: we will use the same batch-size, momentum and number of epochs as the baselines in our experiments (unless specified otherwise).
In addition, we point out again that when $\gamma_t=1$, we recover the SGD step with Nesterov momentum.

In sections \ref{sec:appendix_simplified_fw_algo} and \ref{sec:appendix_smoothing} of the appendix, we detail the derivation of the optimal step-size (step \ref{algo_line:step_size}) and the computation of the search direction (step \ref{algo_line:dual_search}).
The computation of the dual search direction is omitted here for space reasons.
However, its implementation is straightforward in practice, and its computational cost is linear in the size of the output space.

Finally, we emphasize that the DFW algorithm is motivated by an empirical perspective.
While our method is not guaranteed to converge, our experiments show an effective minimization of the learning objective for the problems encountered in practice.

\section{Experiments}

We compare the Deep Frank Wolfe (DFW) algorithm to the state-of-the-art optimizers.
We show that, across diverse data sets and architectures, the DFW algorithm outperforms adaptive gradient methods (with the exception of one setting, DN-10, where it obtains similar performance to AMSGrad and BPGrad).
In addition, the DFW algorithm offers competitive and sometimes superior performance to SGD at a lower computational cost, even though SGD has the advantage of a hand-designed schedule that has been chosen separately for each of these tasks.

Our experiments are implemented in pytorch \citep{Paszke2017}, and the code is available at \url{https://github.com/oval-group/dfw}.
All models are trained on a single Nvidia Titan Xp card.

\subsection{Image Classification with Convolutional Neural Networks}

\paragraph{Data Set \& Architectures.} The CIFAR-10/100 data sets contain 60,000 RGB natural images of size 32 $\times$ 32 with 10/100 classes \citep{Krizhevsky2009}.
We split the training set into 45,000 training samples and 5,000 validation samples, and use 10,000 samples for testing.
The images are centered and normalized per channel.
Unless specified otherwise, no data augmentation is employed.
We perform our experiments on two modern architectures of deep convolutional neural networks: wide residual networks \citep{Zagoruyko2016}, and densely connected convolutional networks \citep{Huang2017a}.
Specifically, we employ a wide residual network of depth 40 and width factor 4, which has 8.9M parameters, and a \say{bottleneck} densely connected convolutional neural network of depth 40 and growth factor 40, which has 1.9M parameters.
We refer to these architectures as WRN and DN respectively.
All the following experimental details follow the protocol of \citep{Zagoruyko2016} and \citep{Huang2017a}.
The only difference is that, instead of using 50,000 samples for training, we use 45,000 samples for training, and 5,000 samples for the validation set, which we found to be essential for all adaptive methods.
While Deep Frank Wolfe (DFW) uses an SVM loss, the baselines are trained with the Cross-Entropy (CE) loss since this resulted in better performance.

\paragraph{Method.} We compare DFW to the most common adaptive learning rates currently used: Adagrad \citep{Duchi2011}, Adam \citep{Kingma2015}, the corrected version of Adam called AMSGrad \citep{Reddi2018}, and BPGrad \citep{Zhang2017a}.
For these methods and for DFW, we cross-validate the initial learning rate as a power of 10.
We also evaluate the performance of SGD with momentum (simply referred to as SGD), for which we follow the protocol of \citep{Zagoruyko2016} and \citep{Huang2017a}.
For all methods, we set a budget of 200 epochs for WRN and 300 epochs for DN.
Furthermore, the batch-size is respectively set to 128 and 64 for WRN and DN as in \citep{Zagoruyko2016} and \citep{Huang2017a}.
For DN, the $l_2$ regularization is set to $10^{-4}$ as in \citep{Huang2017a}.
For WRN, the $l_2$ is cross-validated between $5.10^{-4}$, as in \citep{Zagoruyko2016}, and $10^{-4}$, a more usual value that we have found to perform better for some of the methods (in particular DFW, since the corresponding loss function is an SVM instead of CE, for which the value of $5.10^{-4}$ was designed).
The value of the Nesterov momentum is set to 0.9 for BPGrad, SGD and DFW.
DFW has only one hyper-parameter to tune, namely $\eta$, which is analogous to an initial learning rate.
For SGD, the initial learning rate is set to 0.1 on both WRN and DN.
Following \citep{Zagoruyko2016} and \citep{Huang2017a}, it is then divided by 5 at epochs 60, 120 and 180 for WRN, and by 10 at epochs 150 and 225 for DN.

\paragraph{Results.} We present the results in Table \ref{tab:cifar}.
\begin{table}[h]
\centering
\begin{tabular}{clcc}
\multirow{2}{*}{Architecture} & \multirow{2}{*}{Optimizer} & CIFAR-10           & CIFAR-100  \\
                              &                            & Test Accuracy (\%) & Test Accuracy (\%) \\
\hline \hline
\multirow{6}{*}{WRN}       & Adagrad                    & 86.07 & 57.64    \\
                           & Adam                       & 84.86 & 58.46    \\
                           & AMSGrad                    & 86.08 & 60.73    \\
                           & BPGrad                     & 88.62 & 60.31    \\
                           & DFW                        & {\bf 90.18} & {\bf 67.83}  \\
                           & {\color{red} SGD}          & {\color{red} 90.08}  & {\color{red} 66.78}  \\ \hline

\multirow{6}{*}{DN}         & Adagrad                    & 87.32 & 56.47    \\
                            & Adam                       & 88.44 & 64.61    \\
                            & AMSGrad                    & 90.53 & 68.32    \\
                            & BPGrad                     & {\bf 90.85} & 59.36   \\
                            & DFW                        & 90.22 & {\bf 69.55}   \\
                            & {\color{red} SGD}          & {\color{red} \bf 92.02} & {\color{red} \bf 70.33}    \\
\end{tabular}
\caption{\em Results on the CIFAR data sets without data augmentation.
    In black, all adaptive methods have a single hyper-parameter for their step-size.
    In red, SGD benefits from a hand-designed schedule.
    DFW outperforms all baselines on the WRN architecture, by a margin of 7\% for adaptive gradient methods on CIFAR-100.
    On the DN-100 task, DFW exceeds the accuracy of Adagrad by 14\%.}
\label{tab:cifar}
\end{table}

Observe that DFW significantly outperforms the adaptive gradient methods, particularly on the more challenging CIFAR-100 data set.
On the WRN-CIFAR-100 task in particular, DFW obtains a testing accuracy which is about 7\% higher than all other adaptive methods and outperforms SGD with a hand-designed schedule by 1\%.
The inferior generalization of adaptive gradient methods is consistent with the findings of \cite{Wilson2017}.
On all tasks, the accuracy of DFW is comparable to SGD. Note that DFW converges significantly faster than SGD: the network reaches its final performance several times faster than SGD in all cases.
We illustrate this with an example in figure \ref{fig:error_cifar}, which plots the training and validation errors on DN-CIFAR-100.
In figure \ref{fig:gamma_cifar}, one can see how the step-size is automatically decayed by DFW on this same experiment: we compare the effective step-size $\gamma_t \eta$ for DFW to the manually tuned $\eta_t$ for SGD.

\begin{figure}[h]
\centering
\begin{minipage}{.47\textwidth}
  \centering
  \includegraphics[height=0.65\linewidth]{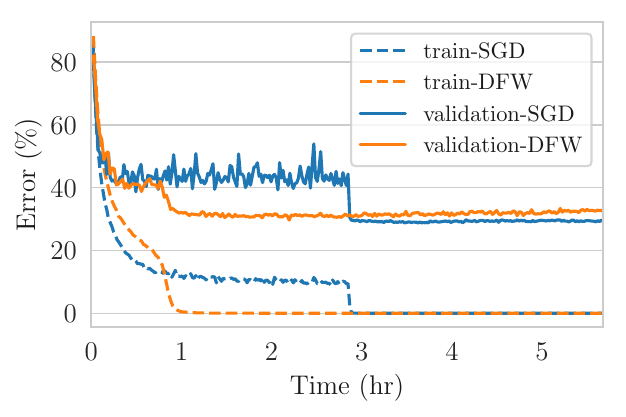}
  \captionof{figure}{\em Training and validation error during the training of DN on CIFAR-100. DFW converges significantly faster than SGD.}
  \label{fig:error_cifar}
\end{minipage}
\hfill
\begin{minipage}{.47\textwidth}
  \centering
  \includegraphics[height=0.65\linewidth]{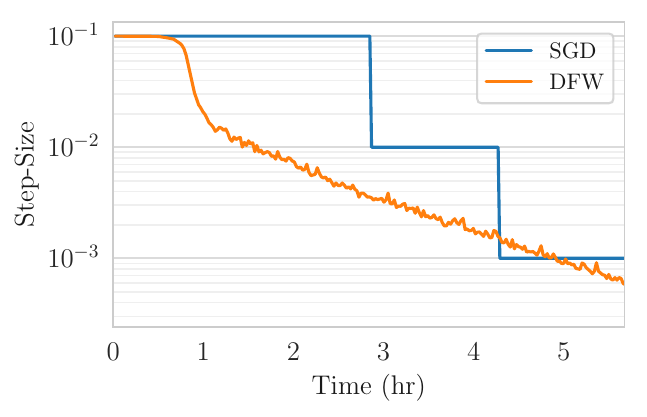}
  \captionof{figure}{\em The (automatic) evolution of $\gamma_t \eta$ for the DFW algorithm compared to the "staircase" hand-designed schedule of $\eta_t$ for SGD.}
  \label{fig:gamma_cifar}
\end{minipage}
\end{figure}

\paragraph{Data Augmentation.}
Since data augmentation provides a significant boost to the final accuracy, we provide additional results that make use of it.
Specifically, we randomly flip the images horizontally and randomly crop them with four pixels padding.
For methods that do not use a hand-designed schedule, such data augmentation introduces additional variance which makes the adaptation of the step-size more difficult.
Therefore we allow the batch size of adaptive methods (e.g. all methods but SGD) to be chosen as 1x, 2x or 4x, where x is the original value of batch-size (64 for DN, 128 for WRN).
Due to the heavy computational cost of the cross-validation (we tune the batch-size, regularization and initial learning rate), we provide results for SGD, DFW and the best performing adaptive gradient method, which is AMSGrad.
For SGD the hyper-parameters are kept the same as in \citep{Zagoruyko2016} and \citep{Huang2017a}.
We present the results in Table \ref{tab:cifar_augmented}.

\begin{table}[H]
\centering
\begin{tabular}{clcc}
\multirow{2}{*}{Architecture} & \multirow{2}{*}{Optimizer} & CIFAR-10                 & CIFAR-100   \\
                              &                            & Test Accuracy (\%)       & Test Accuracy (\%) \\
\hline \hline
\multirow{4}{*}{WRN}                           & AMSGrad                    & 90.06                    & 67.75 \\
                           & DFW                        & {\bf 94.71}              & {\bf 74.71} \\
                           & {\color{red} SGD}          & {\bf \color{red} 95.40}  & {\bf \color{red} 77.78} \\
                           & {\color{red} SGD$^*$}          & {\bf \color{red} 95.47}  & {\bf \color{red} 78.82} \\ \hline

\multirow{3}{*}{DN}        & AMSGrad                    & 91.78                    & 69.58 \\
                            & DFW                        & {\bf 94.88}              & {\bf 73.20} \\
                            & {\color{red} SGD}          & {\bf \color{red} 95.26}  & {\bf \color{red} 76.26}
\end{tabular}
\caption{\em Results on the CIFAR data sets with data augmentation.
    In black, all adaptive methods have a single hyper-parameter for their step-size.
    In red, SGD benefits from a hand-designed schedule.
    On the fourth line, SGD$^*$ refers to the result reported in Table 5 of \citep{Huang2017a}.
    The small difference between the results of SGD and SGD$^*$ can be explained by the fact that we use 5,000 fewer training samples in our experiments (these are kept for validation).
    The results of this table show that DFW systematically outperforms AMSGrad on this task (by up to 7\% on WRN-100).}
\label{tab:cifar_augmented}
\end{table}

These results confirm that DFW consistently outperforms AMSGrad, which is the best adaptive baseline on these tasks.
In particular, DFW obtains a test accuracy which is 7\% better than AMSGrad on WRN-100.

\subsection{Natural Language Inference with Recurrent Neural Networks}

\paragraph{Data Set.} The Stanford Natural Language Inference (SNLI) data set is a large corpus of 570k pairs of sentences \citep{Bowman2015}.
Each sentence is labeled by one of the three possible labels: entailment, neutral and contradiction.
This allows the model to learn the semantics of the text data from a three-way classification problem.
Thanks to its scale and its supervised labels, this data set allows large neural networks to learn high-quality text embeddings.
As \cite{Conneau2017} demonstrate, the SNLI corpus can thus be used as a basis for transfer learning in natural language processing, in the same way that the ImageNet data set is used for pre-training in computer vision.

\paragraph{Method.} We follow the protocol of \citep{Conneau2017} to learn their best model, namely a bi-directional LSTM of about 47M parameters.
In particular, the reported results use SGD with an initial learning rate of 0.1 and a hand-designed schedule that adapts to the variations of the validation set: if the validation accuracy does not improve, the learning rate is divided by a factor of 5.
We also report results on Adagrad, Adam, AMSGrad and BPGrad.
Following the official SGD baseline, Nesterov momentum is deactivated.
Using their open-source implementation, we replace the optimization by the DFW algorithm, the CE loss by an SVM, and leave all other components unchanged.
In this experiment, we use the conditional gradient direction rather than the CE gradient, since three-way classification does not cause sparsity in the derivative of the hinge loss (which is the issue that originally motivated our use of a different direction).
We cross-validate our initial proximal term as a power of ten, and do not manually tune any schedule.
In order to disentangle the importance of the loss function from the optimization algorithm, we run the baselines with both an SVM loss and a CE loss. The initial learning rate of the baselines is also cross-validated as a power of ten.

\paragraph{Results.} The results are presented in Table \ref{tab:slni}.
\begin{table}[H]
\centering
\begin{tabular}{lcccccccc}
Optimizer                           & Loss& Adagrad & Adam & AMSGrad & BPGrad & DFW & {\color{red} SGD} & {\color{red} SGD$^*$}  \\ \hline
\multirow{2}{*}{Test Accuracy (\%)} & CE  & 83.8    & 84.5 & 84.2    &  83.6  &  - & {\color{red} 84.7} & {\color{red} 84.5} \\
                                    & SVM & 84.6    & 85.0 & 85.1    &  84.2  & {\bf 85.2}& {\bf \color{red} 85.2}&  -   \\
\end{tabular}
\caption{\em Results on the Stanford Natural Language Inference corpus.
    In black, all adaptive methods have a single hyper-parameter for their step-size.
    In red, SGD benefits from a hand-designed schedule.
    SGD$^*$ refers to the result reported in \citep{Conneau2017}.
    The other results have been obtained with their open-source implementation in our own experiments.}
\label{tab:slni}
\end{table}

Note that these results outperform the reported testing accuracy of 84.5\% in \citep{Conneau2017} that is obtained with CE.
This experiment, which is performed on a completely different architecture and data set than the previous one, confirms that DFW outperforms adaptive gradient methods and matches the performance of SGD with a hand-designed learning rate schedule.

\section{The Importance of The Step-Size}

\subsection{Impact on Generalization}

It is worth discussing the subtle relationship between optimization and generalization.
In order to emphasize the impact of implicit regularization, all results presented in this section do not use data augmentation.
As a first illustrative example, we consider the following experiment: we take the protocol to train the DN network on CIFAR-100 with SGD, and simply change the initial learning rate to be ten times smaller, and the budget of epochs to be ten times larger.
As a result, the final training objective significantly decreases from 0.33 to 0.069.
Yet at the same time, the best validation accuracy decreases from 70.94\% to 68.7\%.
A similar effect occurs when decreasing the value of the momentum, and we have observed this across various convolutional architectures.
In other words, accurate optimization is less important for generalization than the implicit regularization of a high learning rate.

We have observed DFW to accurately optimize the learning objective in our experiments.
However, given the above observation, we believe that its good generalization properties are rather due to its capability to usually maintain a high learning rate at an early stage.
Similarly, the good generalization performance of SGD may be due to its schedule with a large number of steps at a high learning rate.

\subsection{Sensitivity Analysis}

The previous section has qualitatively hinted at the importance of the step-size for generalization.
Here we quantitatively analyze the impact of the initial learning rate $\eta$ on both the training accuracy (quality of optimization) and the validation accuracy (quality of generalization).
We compare results of the DFW and SGD algorithms on the CIFAR data sets when varying the value of $\eta$ as a power of 10.
The results on the validation set are summarized in figure \ref{fig:sensitivity}, and the performance on the training set is reported in Appendix \ref{sec:details_cifar}.

\begin{figure}[h]
\centering
\begin{minipage}{.5\textwidth}
  \centering
\includegraphics[width=0.8\linewidth]{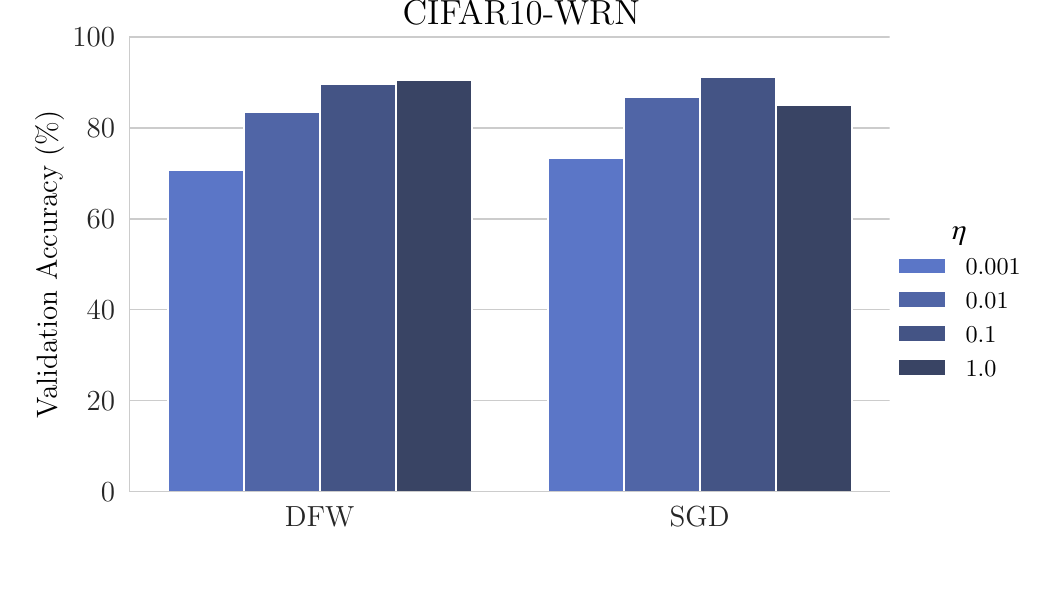}
\includegraphics[width=0.8\linewidth]{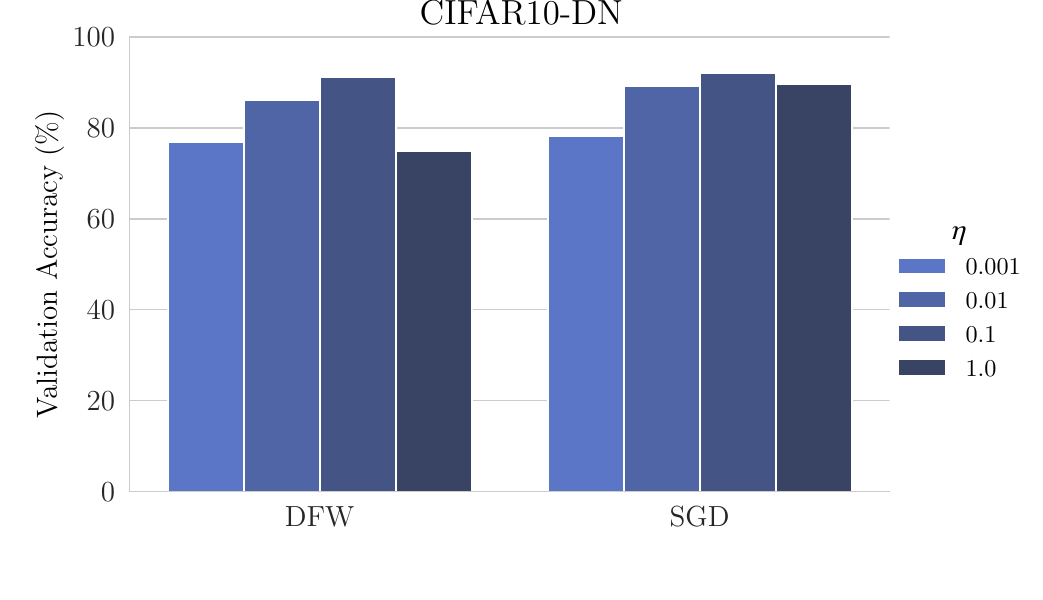}
\end{minipage}%
\hfill
\begin{minipage}{.5\textwidth}
  \centering
  \includegraphics[width=0.8\linewidth]{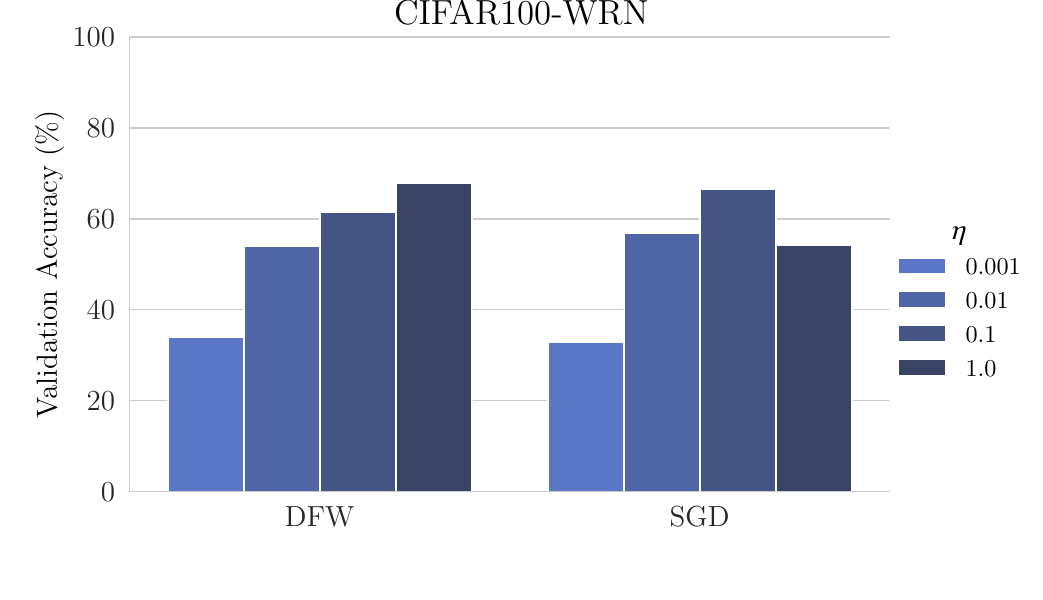}
  \includegraphics[width=0.8\linewidth]{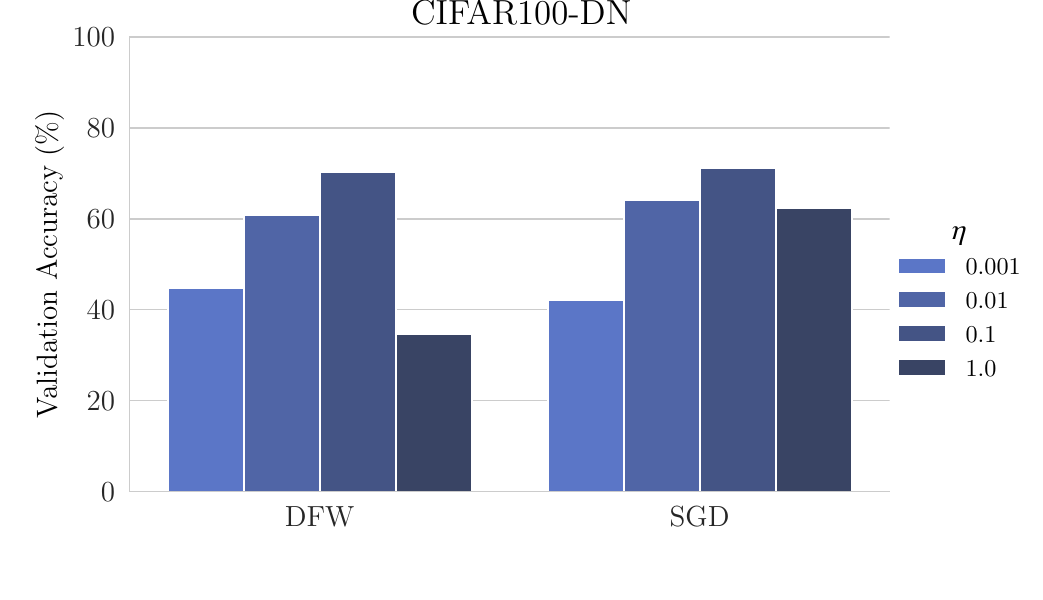}
\end{minipage}
\caption{\em Visualization of the sensitivity analysis for the choice of initial learning rate $\eta$ on the CIFAR data sets.
Each subplot displays the best validation accuracy for DFW and SGD.
Similar plots are available in larger format in Appendix \ref{sec:sensitivity_cifar}.}
\label{fig:sensitivity}
\end{figure}

On the training set, both methods obtain nearly perfect accuracy across at least three orders of magnitude of $\eta$ (details in Appendix \ref{sec:sensitivity_cifar}).
In contrast, the results of figure \ref{fig:sensitivity} confirm that the validation performance is sensitive to the choice of $\eta$ for both methods.

In some cases where $\eta$ is high, SGD obtains a better performance than DFW.
This is because the hand-designed schedule of SGD enforces a decay of $\eta$, while the DFW algorithm relies on an automatic decay of the step-size $\gamma_t$ for effective convergence.
This automatic decay may not happen if a small proximal term (large $\eta$) is combined with a local approximation that is not sufficiently accurate (for instance with a small batch-size).

However, if we allow the DFW algorithm to use a larger batch size, then the local approximation becomes more accurate and it can handle large values of $\eta$ as well.
Interestingly, choosing a larger batch-size and a larger value of $\eta$ can result in better generalization.
For instance, by using a batch-size of 256 (instead of 64) and $\eta=1$, DFW obtains a test accuracy of 72.64\% on CIFAR-100 with the DN architecture (SGD obtains 70.33\% with the settings of \citep{Huang2017a}).

\subsection{Discussion}

Our empirical evidence indicates that the initial learning rate can be a crucial hyper-parameter for good generalization.
We have observed in our experiments that such a choice of high learning rate provides a consistent improvement for convolutional neural networks: accurate minimization of the training objective with large initial steps usually leads to good generalization.
Furthermore, as mentioned in the previous section, it is sometimes beneficial to even increase the batch-size in order to be able to train the model using large initial steps.

In the case of recurrent neural networks, however, this effect is not as distinct.
Additional experiments on different recurrent architectures have showed variations in the impact of the learning rate and in the best-performing optimizer.
Further analysis would be required to understand the effects at play.

\section{Conclusion}
\label{sec:conclusion}

We have introduced DFW, an efficient algorithm to train deep neural networks.
DFW predominantly outperforms adaptive gradient methods, and obtains similar performance to SGD without requiring a hand-designed learning rate schedule.

We emphasize the generality of our framework in Section \ref{sec:problem_formulation}, which enables the training of deep neural networks to benefit from any advance on optimization algorithms for linear SVMs.
This framework could also be applied to other loss functions that yield efficiently solvable proximal problems.
In particular, our algorithm already supports the use of structured prediction loss functions \citep{Taskar2003,Tsochantaridis2004}, which can be used, for instance, for image segmentation.

We have mentioned the intricate relationship between optimization and generalization in deep learning.
This illustrates a major difficulty in the design of effective optimization algorithms for deep neural networks: the learning objective does not include all the regularization needed for good generalization.
We believe that in order to further advance optimization for deep neural networks, it is essential to alleviate this problem and expose a clear objective function to optimize.

\subsubsection*{Acknowledgments}

This work was supported by the EPSRC grants AIMS CDT EP/L015987/1, Seebibyte EP/M013774/1, EP/P020658/1 and TU/B/000048, and by Yougov.
We also thank the Nvidia Corporation for the GPU donation.

\vfill
\pagebreak

\bibliographystyle{iclr2019_conference}
\bibliography{\string~/workspace/bibliography/standardstrings,\string~/workspace/bibliography/oval}

\begin{thebibliography}{59}
\providecommand{\natexlab}[1]{#1}
\providecommand{\url}[1]{\texttt{#1}}
\expandafter\ifx\csname urlstyle\endcsname\relax
  \providecommand{\doi}[1]{doi: #1}\else
  \providecommand{\doi}{doi: \begingroup \urlstyle{rm}\Url}\fi

\bibitem[Andrychowicz et~al.(2016)Andrychowicz, Denil, Gomez, Hoffman, Pfau,
  Schaul, Shillingford, and De~Freitas]{Andrychowicz2016}
Marcin Andrychowicz, Misha Denil, Sergio Gomez, Matthew~W Hoffman, David Pfau,
  Tom Schaul, Brendan Shillingford, and Nando De~Freitas.
\newblock Learning to learn by gradient descent by gradient descent.
\newblock \emph{Neural Information Processing Systems}, 2016.

\bibitem[Arpit et~al.(2017)Arpit, Jastrz{\k{e}}bski, Ballas, Krueger, Bengio,
  Kanwal, Maharaj, Fischer, Courville, Bengio, and Lacoste-Julien]{Arpit2017}
Devansh Arpit, Stanis{\l}aw Jastrz{\k{e}}bski, Nicolas Ballas, David Krueger,
  Emmanuel Bengio, Maxinder~S. Kanwal, Tegan Maharaj, Asja Fischer, Aaron
  Courville, Yoshua Bengio, and Simon Lacoste-Julien.
\newblock A closer look at memorization in deep networks.
\newblock \emph{International Conference on Machine Learning}, 2017.

\bibitem[Ba et~al.(2017)Ba, Grosse, and Martens]{Ba2017}
Jimmy Ba, Roger Grosse, and James Martens.
\newblock Distributed second-order optimization using kronecker-factored
  approximations.
\newblock \emph{International Conference on Learning Representations}, 2017.

\bibitem[Bach(2015)]{Bach2015}
Francis Bach.
\newblock Duality between subgradient and conditional gradient methods.
\newblock \emph{SIAM Journal on Optimization}, 2015.

\bibitem[Baydin et~al.(2018)Baydin, Cornish, Rubio, Schmidt, and
  Wood]{Baydin2018}
Atilim~Gunes Baydin, Robert Cornish, David~Martinez Rubio, Mark Schmidt, and
  Frank Wood.
\newblock Online learning rate adaptation with hypergradient descent.
\newblock \emph{International Conference on Learning Representations}, 2018.

\bibitem[Berrada et~al.(2017)Berrada, Zisserman, and Kumar]{Berrada2017}
Leonard Berrada, Andrew Zisserman, and M~Pawan Kumar.
\newblock Trusting {SVM} for piecewise linear {CNNs}.
\newblock \emph{International Conference on Learning Representations}, 2017.

\bibitem[Berrada et~al.(2018)Berrada, Zisserman, and Kumar]{Berrada2018}
Leonard Berrada, Andrew Zisserman, and M~Pawan Kumar.
\newblock Smooth loss functions for deep top-k classification.
\newblock \emph{International Conference on Learning Representations}, 2018.

\bibitem[Botev et~al.(2017)Botev, Ritter, and Barber]{Botev2017}
Aleksandar Botev, Hippolyt Ritter, and David Barber.
\newblock Practical gauss-newton optimisation for deep learning.
\newblock \emph{International Conference on Machine Learning}, 2017.

\bibitem[Bowman et~al.(2015)Bowman, Angeli, Potts, and Manning]{Bowman2015}
Samuel~R Bowman, Gabor Angeli, Christopher Potts, and Christopher~D Manning.
\newblock A large annotated corpus for learning natural language inference.
\newblock \emph{Conference on Empirical Methods in Natural Language
  Processing}, 2015.

\bibitem[Bubeck(2015)]{Bubeck2015}
S{\'e}bastien Bubeck.
\newblock Convex optimization: Algorithms and complexity.
\newblock \emph{Foundations and Trends in Machine Learning}, 2015.

\bibitem[Chaudhari \& Soatto(2018)Chaudhari and Soatto]{Chaudhari2018}
Pratik Chaudhari and Stefano Soatto.
\newblock Stochastic gradient descent performs variational inference, converges
  to limit cycles for deep networks.
\newblock \emph{International Conference on Learning Representations}, 2018.

\bibitem[Conneau et~al.(2017)Conneau, Kiela, Schwenk, Barrault, and
  Bordes]{Conneau2017}
Alexis Conneau, Douwe Kiela, Holger Schwenk, Loic Barrault, and Antoine Bordes.
\newblock Supervised learning of universal sentence representations from
  natural language inference data.
\newblock \emph{Conference on Empirical Methods in Natural Language
  Processing}, 2017.

\bibitem[Desjardins et~al.(2015)Desjardins, Simonyan, Pascanu,
  et~al.]{Desjardins2015}
Guillaume Desjardins, Karen Simonyan, Razvan Pascanu, et~al.
\newblock Natural neural networks.
\newblock \emph{Neural Information Processing Systems}, 2015.

\bibitem[Duchi et~al.(2011)Duchi, Hazan, and Singer]{Duchi2011}
John Duchi, Elad Hazan, and Yoram Singer.
\newblock Adaptive subgradient methods for online learning and stochastic
  optimization.
\newblock \emph{Journal of Machine Learning Research}, 2011.

\bibitem[Frank \& Wolfe(1956)Frank and Wolfe]{Frank1956}
Marguerite Frank and Philip Wolfe.
\newblock An algorithm for quadratic programming.
\newblock \emph{Naval Research Logistics Quarterly}, 1956.

\bibitem[Frerix et~al.(2018)Frerix, Möllenhoff, Moeller, and
  Cremers]{Frerix2018}
Thomas Frerix, Thomas Möllenhoff, Michael Moeller, and Daniel Cremers.
\newblock Proximal backpropagation.
\newblock \emph{International Conference on Learning Representations}, 2018.

\bibitem[Goel et~al.(2017)Goel, Kanade, Klivans, and Thaler]{Goel2017}
Surbhi Goel, Varun Kanade, Adam Klivans, and Justin Thaler.
\newblock Reliably learning the {ReLU} in polynomial time.
\newblock \emph{Conference on Learning Theory}, 2017.

\bibitem[Grosse \& Martens(2016)Grosse and Martens]{Grosse2016}
Roger Grosse and James Martens.
\newblock A kronecker-factored approximate fisher matrix for convolution
  layers.
\newblock \emph{International Conference on Machine Learning}, 2016.

\bibitem[Hardt et~al.(2016)Hardt, Recht, and Singer]{Hardt2016}
Moritz Hardt, Benjamin Recht, and Yoram Singer.
\newblock Train faster, generalize better: Stability of stochastic gradient
  descent.
\newblock \emph{International Conference on Machine Learning}, 2016.

\bibitem[He et~al.(2016)He, Zhang, Ren, and Sun]{He2016}
Kaiming He, Xiangyu Zhang, Shaoqing Ren, and Jian Sun.
\newblock Deep residual learning for image recognition.
\newblock \emph{Conference on Computer Vision and Pattern Recognition}, 2016.

\bibitem[Heinemann et~al.(2016)Heinemann, Livni, Eban, Elidan, and
  Globerson]{Heinemann2016}
Uri Heinemann, Roi Livni, Elad Eban, Gal Elidan, and Amir Globerson.
\newblock Improper deep kernels.
\newblock \emph{International Conference on Artificial Intelligence and
  Statistics}, 2016.

\bibitem[Hochreiter \& Obermayer(2005)Hochreiter and Obermayer]{Hochreiter2005}
Sepp Hochreiter and Klaus Obermayer.
\newblock Optimal gradient-based learning using importance weights.
\newblock \emph{International Joint Conference on Neural Networks}, 2005.

\bibitem[Hoffer et~al.(2017)Hoffer, Hubara, and Soudry]{Hoffer2017}
Elad Hoffer, Itay Hubara, and Daniel Soudry.
\newblock Train longer, generalize better: closing the generalization gap in
  large batch training of neural networks.
\newblock \emph{Neural Information Processing Systems}, 2017.

\bibitem[Huang et~al.(2017)Huang, Liu, Weinberger, and van~der
  Maaten]{Huang2017a}
Gao Huang, Zhuang Liu, Kilian~Q Weinberger, and Laurens van~der Maaten.
\newblock Densely connected convolutional networks.
\newblock \emph{Conference on Computer Vision and Pattern Recognition}, 2017.

\bibitem[Kingma \& Ba(2015)Kingma and Ba]{Kingma2015}
Diederik~P. Kingma and Jimmy Ba.
\newblock Adam: A method for stochastic optimization.
\newblock \emph{International Conference on Learning Representations}, 2015.

\bibitem[Krizhevsky(2009)]{Krizhevsky2009}
Alex Krizhevsky.
\newblock Learning multiple layers of features from tiny images.
\newblock \emph{Technical Report}, 2009.

\bibitem[Lacoste-Julien et~al.(2013)Lacoste-Julien, Jaggi, Schmidt, and
  Pletscher]{Lacoste-Julien2013}
Simon Lacoste-Julien, Martin Jaggi, Mark Schmidt, and Patrick Pletscher.
\newblock Block-coordinate {F}rank-{W}olfe optimization for structural {SVMs}.
\newblock \emph{International Conference on Machine Learning}, 2013.

\bibitem[Lewis \& Wright(2016)Lewis and Wright]{Lewis2016}
Adrian~S Lewis and Stephen~J Wright.
\newblock A proximal method for composite minimization.
\newblock \emph{Mathematical Programming}, 2016.

\bibitem[Li \& Malik(2017)Li and Malik]{Li2017b}
Ke~Li and Jitendra Malik.
\newblock Learning to optimize.
\newblock \emph{International Conference on Learning Representations}, 2017.

\bibitem[Loshchilov \& Hutter(2017)Loshchilov and Hutter]{Loshchilov2017}
Ilya Loshchilov and Frank Hutter.
\newblock {SGDR}: Stochastic gradient descent with warm restarts.
\newblock \emph{International Conference on Learning Representations}, 2017.

\bibitem[Martens \& Grosse(2015)Martens and Grosse]{Martens2015}
James Martens and Roger Grosse.
\newblock Optimizing neural networks with {Kronecker}-factored approximate
  curvature.
\newblock \emph{International Conference on Machine Learning}, 2015.

\bibitem[Martens \& Sutskever(2012)Martens and Sutskever]{Martens2012}
James Martens and Ilya Sutskever.
\newblock Training deep and recurrent networks with {H}essian-free
  optimization.
\newblock \emph{Neural Networks: Tricks of the Trade}, 2012.

\bibitem[Martens et~al.(2018)Martens, Ba, and Johnson]{Martens2018}
James Martens, Jimmy Ba, and Matt Johnson.
\newblock Kronecker-factored curvature approximations for recurrent neural
  networks.
\newblock \emph{International Conference on Learning Representations}, 2018.

\bibitem[Mohapatra et~al.(2016)Mohapatra, Dokania, Jawahar, and
  Kumar]{Mohapatra2016}
Pritish Mohapatra, Puneet Dokania, C.~V. Jawahar, and M.~Pawan Kumar.
\newblock Partial linearization based optimization for multi-class {SVM}.
\newblock \emph{European Conference on Computer Vision}, 2016.

\bibitem[Nesterov(1983)]{Nesterov1983}
Yurii Nesterov.
\newblock A method of solving a convex programming problem with convergence
  rate $\mathcal{O}(1/k^2)$.
\newblock \emph{Soviet Mathematics Doklady}, 1983.

\bibitem[Neyshabur et~al.(2015)Neyshabur, Salakhutdinov, and
  Srebro]{Neyshabur2015}
Behnam Neyshabur, Ruslan~R Salakhutdinov, and Nati Srebro.
\newblock Path-sgd: Path-normalized optimization in deep neural networks.
\newblock \emph{Neural Information Processing Systems}, 2015.

\bibitem[Neyshabur et~al.(2016)Neyshabur, Wu, Salakhutdinov, and
  Srebro]{Neyshabur2016}
Behnam Neyshabur, Yuhuai Wu, Ruslan~R Salakhutdinov, and Nati Srebro.
\newblock Path-normalized optimization of recurrent neural networks with relu
  activations.
\newblock \emph{Neural Information Processing Systems}, 2016.

\bibitem[Neyshabur et~al.(2017)Neyshabur, Bhojanapalli, McAllester, and
  Srebro]{Neyshabur2017}
Behnam Neyshabur, Srinadh Bhojanapalli, David McAllester, and Nati Srebro.
\newblock Exploring generalization in deep learning.
\newblock \emph{Neural Information Processing Systems}, 2017.

\bibitem[Ollivier(2013)]{Ollivier2013}
Yann Ollivier.
\newblock Riemannian metrics for neural networks.
\newblock \emph{Information and Inference: a Journal of the IMA}, 2013.

\bibitem[Paszke et~al.(2017)Paszke, Gross, Chintala, Chanan, Yang, DeVito, Lin,
  Desmaison, Antiga, and Lerer]{Paszke2017}
Adam Paszke, Sam Gross, Soumith Chintala, Gregory Chanan, Edward Yang, Zachary
  DeVito, Zeming Lin, Alban Desmaison, Luca Antiga, and Adam Lerer.
\newblock Automatic differentiation in pytorch.
\newblock \emph{NIPS Autodiff Workshop}, 2017.

\bibitem[Ravi \& Larochelle(2017)Ravi and Larochelle]{Ravi2017}
Sachin Ravi and Hugo Larochelle.
\newblock Optimization as a model for few-shot learning.
\newblock \emph{International Conference on Learning Representations}, 2017.

\bibitem[Reddi et~al.(2018)Reddi, Kale, and Kumar]{Reddi2018}
Sashank~J Reddi, Satyen Kale, and Sanjiv Kumar.
\newblock On the convergence of adam and beyond.
\newblock \emph{International Conference on Learning Representations}, 2018.

\bibitem[Roux et~al.(2008)Roux, Manzagol, and Bengio]{Roux2008}
Nicolas~L Roux, Pierre-Antoine Manzagol, and Yoshua Bengio.
\newblock Topmoumoute online natural gradient algorithm.
\newblock \emph{Neural Information Processing Systems}, 2008.

\bibitem[Rumelhart et~al.(1986)Rumelhart, Hinton, and Williams]{Rumelhart1986}
David Rumelhart, Geoffrey Hinton, and Ronald Williams.
\newblock Learning representations by back-propagating errors.
\newblock \emph{Nature}, 1986.

\bibitem[Schaul et~al.(2013)Schaul, Zhang, and LeCun]{Schaul2013}
Tom Schaul, Sixin Zhang, and Yann LeCun.
\newblock No more pesky learning rates.
\newblock \emph{International Conference on Machine Learning}, 2013.

\bibitem[Simonyan \& Zisserman(2015)Simonyan and Zisserman]{Simonyan2015}
Karen Simonyan and Andrew Zisserman.
\newblock Very deep convolutional networks for large-scale image recognition.
\newblock \emph{International Conference on Learning Representations}, 2015.

\bibitem[Singh \& Shawe-Taylor(2018)Singh and Shawe-Taylor]{Singh2018}
Gaurav Singh and John Shawe-Taylor.
\newblock Faster convergence \& generalization in {DNN}s.
\newblock \emph{arXiv preprint}, 2018.

\bibitem[Szegedy et~al.(2015)Szegedy, Liu, Jia, Sermanet, Reed, Anguelov,
  Erhan, Vanhoucke, Rabinovich, et~al.]{Szegedy2015}
Christian Szegedy, Wei Liu, Yangqing Jia, Pierre Sermanet, Scott Reed, Dragomir
  Anguelov, Dumitru Erhan, Vincent Vanhoucke, Andrew Rabinovich, et~al.
\newblock Going deeper with convolutions.
\newblock \emph{Conference on Computer Vision and Pattern Recognition}, 2015.

\bibitem[Taskar et~al.(2003)Taskar, Guestrin, and Koller]{Taskar2003}
Benjamin Taskar, Carlos Guestrin, and Daphne Koller.
\newblock Max-margin {Markov} networks.
\newblock \emph{Neural Information Processing Systems}, 2003.

\bibitem[Taylor et~al.(2016)Taylor, Burmeister, Xu, Singh, Patel, and
  Goldstein]{Taylor2016}
Gavin Taylor, Ryan Burmeister, Zheng Xu, Bharat Singh, Ankit Patel, and Tom
  Goldstein.
\newblock Training neural networks without gradients: A scalable {ADMM}
  approach.
\newblock \emph{International Conference on Machine Learning}, 2016.

\bibitem[Tsochantaridis et~al.(2004)Tsochantaridis, Hofmann, Joachims, and
  Altun]{Tsochantaridis2004}
Ioannis Tsochantaridis, Thomas Hofmann, Thorsten Joachims, and Yasemin Altun.
\newblock Support vector machine learning for interdependent and structured
  output spaces.
\newblock \emph{International Conference on Machine Learning}, 2004.

\bibitem[Wichrowska et~al.(2017)Wichrowska, Maheswaranathan, Hoffman,
  Colmenarejo, Denil, de~Freitas, and Sohl-Dickstein]{Wichrowska2017}
Olga Wichrowska, Niru Maheswaranathan, Matthew~W Hoffman, Sergio~Gomez
  Colmenarejo, Misha Denil, Nando de~Freitas, and Jascha Sohl-Dickstein.
\newblock Learned optimizers that scale and generalize.
\newblock \emph{International Conference on Machine Learning}, 2017.

\bibitem[Wilson et~al.(2017)Wilson, Roelofs, Stern, Srebro, and
  Recht]{Wilson2017}
Ashia~C Wilson, Rebecca Roelofs, Mitchell Stern, Nati Srebro, and Benjamin
  Recht.
\newblock The marginal value of adaptive gradient methods in machine learning.
\newblock \emph{Neural Information Processing Systems}, 2017.

\bibitem[Wu et~al.(2018)Wu, Ward, and Bottou]{Wu2018}
Xiaoxia Wu, Rachel Ward, and L{\'e}on Bottou.
\newblock {WNGrad}: Learn the learning rate in gradient descent.
\newblock \emph{arXiv preprint}, 2018.

\bibitem[Yuille \& Rangarajan(2002)Yuille and Rangarajan]{Yuille2002}
Alan~L. Yuille and Anand Rangarajan.
\newblock The concave-convex procedure ({CCCP}).
\newblock \emph{Neural Information Processing Systems}, 2002.

\bibitem[Zagoruyko \& Komodakis(2016)Zagoruyko and Komodakis]{Zagoruyko2016}
Sergey Zagoruyko and Nikos Komodakis.
\newblock Wide residual networks.
\newblock \emph{British Machine Vision Conference}, 2016.

\bibitem[Zeiler(2012)]{Zeiler2012}
Matthew Zeiler.
\newblock {ADADELTA:} an adaptive learning rate method.
\newblock \emph{arXiv preprint}, 2012.

\bibitem[Zhang et~al.(2017{\natexlab{a}})Zhang, Liang, and
  Wainwright]{Zhang2017}
Yuchen Zhang, Percy Liang, and Martin~J. Wainwright.
\newblock Convexified convolutional neural networks.
\newblock \emph{International Conference on Machine Learning},
  2017{\natexlab{a}}.

\bibitem[Zhang et~al.(2017{\natexlab{b}})Zhang, Wu, and Wang]{Zhang2017a}
Ziming Zhang, Yuanwei Wu, and Guanghui Wang.
\newblock Bpgrad: Towards global optimality in deep learning via branch and
  pruning.
\newblock \emph{Conference on Computer Vision and Pattern Recognition},
  2017{\natexlab{b}}.

\end{thebibliography}
\vfill
\pagebreak

\appendix

\section{Proofs \& Algorithms}

For completeness, we prove results for our specific instance of Structural SVM problem. We point out that the proofs of sections \ref{sec:preliminaries}, \ref{sec:dual_objective} and \ref{sec:optimal_step_size} are adaptations from \citep{Lacoste-Julien2013}. Propositions are numbered according to their appearance in the paper.

\subsection{Preliminaries}
\label{sec:preliminaries}

In this section, we assume the loss $\LL$ to be a hinge loss:
\begin{equation} \label{app:eq:hinge_loss}
\LL_{hinge}: (\uu, y) \in \mathbb{R}^{|\Y|} \times \Y \mapsto \max \left\{ \max_{\ybar \in \Y \backslash \{y\}} \left\{ u_{\ybar} + 1 - u_y \right\}, 0 \right\}
\end{equation}
We suppose that we have received a sample $(\x, y)$. We simplify the notation $\f(\w, \x) = \f_{\x}(\w)$ and $\mathcal{L}(\uu, y) = \mathcal{L}_y(\uu)$. For simplicity of the notation, and without loss of generality, we consider the proximal problem obtained at time $t=0$:
\begin{equation} \label{app:eq:proxfw_step_min}
\min\limits_{\w \in \mathbb{R}^p} \left\{ \dfrac{1}{2 \eta} \|\w - \w_0 \|^2 + \T_{\w_0}\rho(\w) + \LL_y\left(\T_{\w_0}\f_{\x}(\w)\right) \right\}.
\end{equation}

Let us define the classification task loss:
\begin{equation}
\text{For } (\bar{y}, y) \in \Y^2, \Delta(\bar{y}, y) =
\begin{cases}
    0& \text{if } \bar{y} = y,\\
    1              & \text{otherwise}.
\end{cases}
\end{equation}

Using this notation, the multi-class hinge loss can be written as:
\begin{equation}
\LL_{hinge}(\uu, y)= \max_{\ybar \in \Y} \left\{ u_{\ybar} + \Delta(\ybar, y) - u_y \right\}.
\end{equation}
Indeed, we can successively write:
\begin{equation}
\begin{split}
\LL_{hinge}(\uu, y)
    &= \max \left\{ \max_{\ybar \in \Y \backslash \{y\}} \left\{ u_{\ybar} + 1 - u_y \right\}, 0 \right\}, \\
    &= \max_{\ybar \in \Y \backslash \{y\}} \left\{ \max \left\{ u_{\ybar} + 1 - u_y, 0  \right\} \right\}, \\
    &= \max_{\ybar \in \Y \backslash \{y\}} \left\{ \max \left\{ u_{\ybar} + \Delta(\ybar, y) - u_y, 0  \right\} \right\}, \\
    &= \max_{\ybar \in \Y} \left\{ \max \left\{ u_{\ybar} + \Delta(\ybar, y) - u_y, 0  \right\} \right\}, \\
    &= \max_{\ybar \in \Y} \left\{ u_{\ybar} + \Delta(\ybar, y) - u_y \right\}.
\end{split}
\end{equation}

We are now going to re-write problem (\ref{app:eq:proxfw_step_min}) as the sum of a quadratic term and a pointwise maximum of linear functions.
For $\ybar \in \Y$, let us define:
\begin{equation}
\begin{split}
\mathbf{a}_{\ybar}
    &= \evalat{\partial \rho(\w)}{\w_0} + \evalat{\partial f_{\x, \ybar}(\w)}{\w_0} - \evalat{\partial f_{\x, y}(\w)}{\w_0}, \\
b_{\ybar}
    &= \rho(\w_0) + f_{\x, \ybar}(\w_0) - f_{\x, y}(\w_0) + \Delta(\ybar, y).
\end{split}
\end{equation}
Then we have that:
\begin{equation}
\begin{split}
\max\limits_{\ybar \in \Y} \left\{ \mathbf{a}_{\ybar}^\top (\w - \w_0) + b_{\ybar} \right\}
    &= \max\limits_{\ybar \in \Y} \Big\{ (\evalat{\partial \rho(\w)}{\w_0} + \evalat{\partial f_{\x, \ybar}(\w)}{\w_0} - \evalat{\partial f_{\x, y}(\w)}{\w_0})^\top (\w - \w_0) \\
    &\qquad + \rho(\w_0) + f_{\x, \ybar}(\w_0) - f_{\x, y}(\w_0) + \Delta(\ybar, y) \Big\}, \\
    &= \rho(\w_0) + \evalat{\partial \rho(\w)}{\w_0}^\top(\w - \w_0) \\
    &\qquad + \max\limits_{\ybar \in \Y} \left\{ \evalat{\partial f_{\x, \ybar}(\w)}{\w_0}^\top (\w - \w_0) +  f_{\x, \ybar}(\w_0) + \Delta(\ybar, y) \right\} \\
    &\qquad - f_{\x, y}(\w_0) - \evalat{\partial f_{\x, y}(\w)}{\w_0}^\top (\w - \w_0), \\
    &= \T_{\w_0}\rho(\w) + \LL\left(\T_{\w_0}\f_\x(\w), y \right).
\end{split}
\end{equation}
Therefore, problem (\ref{app:eq:proxfw_step_min}) can be written as:
\begin{equation} \label{app:eq:canonical_form}
\min\limits_{\w \in \mathbb{R}^p} \left\{ \dfrac{1}{2 \eta} \|\w - \w_0 \|^2 + \max\limits_{\ybar \in \Y} \left\{ \mathbf{a}_{\ybar}^\top (\w - \w_0) + b_{\ybar} \right\} \right\}.
\end{equation}

We notice that the term $\rho(\w_0)$ in $\bb$ is a constant that does not depend on $\w$ nor $\ybar$, therefore we can simplify the expression of $\bb$ to:
\begin{equation}
b_{\ybar}
    = f_{\x, \ybar}(\w_0) - f_{\x, y}(\w_0) + \Delta(\ybar, y).
\end{equation}

We introduce the following notation:
\begin{align}
\what &= \w - \w_0, \\
\PP &= \{ \aalpha \in \mathbb{R}_+^{|\Y|} : \: \sum\limits_{\ybar \in \Y} \alpha_{\ybar} = 1 \}, \\
A &= (\eta \mathbf{a}_{\ybar})_{\ybar \in \Y} \in \mathbb{R}^{p \times |\Y|}.
\end{align}

We will also use the indicator vector: $\mathds{1}_y \in \mathbb{R}^{|\Y|}$, which is equal to 1 at index $y$ and 0 elsewhere.

\subsection{Dual Objective}
\label{sec:dual_objective}

\begin{lemma}[Dual Objective]
The Lagrangian dual of (\ref{app:eq:proxfw_step_min}) is given by:
\begin{equation}
\max\limits_{\aalpha \in \PP} \left\{ -\dfrac{1}{2 \eta} \| A \aalpha\|^2 + \bb^\top \aalpha \right\}.
\end{equation}
Given the dual variables $\aalpha$, the primal can be computed as $\what = - A \aalpha$.
\end{lemma}
\begin{proof}
We derive the Lagrangian of the primal problem. For that, we write the problem in the following equivalent ways:
\begin{align}
&\min\limits_{\what \in \mathbb{R}^p} \left\{ \dfrac{1}{2 \eta} \|\what \|^2 + \max\limits_{\ybar \in \Y} \left\{ \mathbf{a}_{\ybar}^\top \what + b_{\ybar} \right\} \right\}, \\
&\min\limits_{\substack{\what \in \mathbb{R}^p \\ \xi \in \mathbb{R}}} \left\{ \dfrac{1}{2 \eta} \|\what \|^2 + \xi \right\} \: \text{subject to: } \forall \ybar \in \Y, \: \mathbf{a}_{\ybar}^\top \what + b_{\ybar} \leq \xi, \\
&\min\limits_{\substack{\what \in \mathbb{R}^p \\ \xi \in \mathbb{R}}} \sup\limits_{\aalpha \geq 0} \left\{ \dfrac{1}{2 \eta} \|\what \|^2 + \xi + \sum\limits_{\ybar \in \Y} \alpha_{\ybar} \left( \mathbf{a}_{\ybar}^\top \what + b_{\ybar}  - \xi\right) \right\}, \\
&\sup\limits_{\aalpha \geq 0} \min\limits_{\substack{\what \in \mathbb{R}^p \\ \xi \in \mathbb{R}}} \underbrace{\left\{ \dfrac{1}{2 \eta} \|\what \|^2 + \xi + \sum\limits_{\ybar \in \Y} \alpha_{\ybar} \left( \mathbf{a}_{\ybar}^\top \what + b_{\ybar} - \xi \right) \right\}}_{\Lambda(\what, \xi, \aalpha)} \quad \text{(by strong duality)}. \label{app:eq:dual_lagrangian}
\end{align}
We can now write the KKT conditions of the inner minimization problem:
\begin{equation}
\begin{split}
&\dfrac{\partial \Lambda(\what, \xi, \aalpha)}{\partial \xi} = 0: \quad 1 -  \sum\limits_{\ybar \in \Y} \alpha_{\ybar} = 0, \\
&\dfrac{\partial \Lambda(\what, \xi, \aalpha)}{\partial \what} = \mathbf{0}: \quad \frac{1}{\eta}\what + \sum\limits_{\ybar \in \Y} \alpha_{\ybar} \mathbf{a}_{\ybar} = \mathbf{0}.
\end{split}
\end{equation}
This gives $\aalpha \in \PP$ and $\what = - A \aalpha$, since $A = (\eta \mathbf{a}_{\ybar})_{\ybar \in \Y}$ by definition.
By injecting these constraints in $(\ref{app:eq:dual_lagrangian})$, we obtain:
\begin{equation}
\max\limits_{\aalpha \in \PP} \dfrac{1}{2 \eta} \| A \aalpha\|^2 + - A \aalpha^\top \frac{1}{\eta} A \aalpha + \bb^\top \aalpha,
\end{equation}
which finally gives the desired result.
\end{proof}

\subsection{Derivation of the Optimal Step-Size}
\label{sec:optimal_step_size}

\begin{lemma}[Optimal Step-Size]
Suppose that we make a step in the direction of $\s \in \PP$ in the dual.
We define the corresponding primal variables $\w_\s = -A \s$ and $\lambda_\s = \bb^\top \s$, as well as $\lambda = \bb^\top \aalpha$.
Then the optimal step-size is given by:
\begin{equation}
\gamma
    = \dfrac{ (\w - \w_0 - \w_\s)^\top (\w - \w_0)  + \eta (\lambda_\s - \lambda)}{ \|\w - \w_0 - \w_\s\|^2}.
\end{equation}

\end{lemma}
\begin{proof}
Given the direction $\s$, we take the step $\aalpha + \gamma (\s - \aalpha)$. The new objective is given by:
\begin{equation}
-\dfrac{1}{2 \eta} \| A (\aalpha + \gamma (\s - \aalpha))\|^2 + \bb^\top (\aalpha + \gamma (\s - \aalpha)).
\end{equation}
In order to compute the optimal step-size, we compute the derivative of the above expression with respect to gamma, and set it to 0:
\begin{equation}
-\dfrac{1}{\eta}(\s - \aalpha)^\top A ^\top A (\aalpha + \gamma (\s - \aalpha)) + \bb^\top (\s - \aalpha) = 0.
\end{equation}
We can isolate the unique term containing $\gamma$:
\begin{equation}
-\dfrac{1}{\eta} \gamma \|A(\s - \aalpha)\|^2 -\dfrac{1}{\eta}(\s - \aalpha)^\top A ^\top A \aalpha  + \bb^\top (\s - \aalpha) = 0.
\end{equation}
This yields:
\begin{equation}
\begin{split}
\gamma
    &= \dfrac{-\frac{1}{\eta}(\s - \aalpha)^\top A ^\top A \aalpha  + \bb^\top (\s - \aalpha)}{\frac{1}{\eta} \|A(\s - \aalpha)\|^2}, \\
    &= \dfrac{-\frac{1}{\eta}(A \s - A \aalpha)^\top A \aalpha  + \bb^\top (\s - \aalpha)}{\frac{1}{\eta} \|A \s - A \aalpha\|^2}, \\
    &= \dfrac{-(A \s - A \aalpha)^\top A \aalpha  + \eta \bb^\top (\s - \aalpha)}{ \|A \s - A \aalpha\|^2}.
\end{split}
\end{equation}
We can then inject the primal variables and simplify:
\begin{equation}
\begin{split}
\gamma
    &= \dfrac{(-\w_\s + \what)^\top \what  + \eta (\lambda_{\s} - \lambda)}{\|-\w_\s + \what\|^2}, \\
    &= \dfrac{(\w - \w_0 - \w_\s)^\top (\w - \w_0)  + \eta (\lambda_\s - \lambda)}{ \|\w - \w_0 - \w_\s\|^2}.
\end{split}
\end{equation}
\end{proof}

\subsection{Primal-Dual Proximal Frank-Wolfe Algorithm}
\label{sec:appendix_fw_algo}

We present here the primal-dual algorithm that solves $(\ref{app:eq:proxfw_step_min})$ using the previous results:
\begin{algorithm}[H]
\caption{\em Proximal Frank Wolfe Algorithm}\label{algo:proxfw}
\begin{algorithmic}[1]
\Require proximal coefficient $\eta$, initial point $\w_0 \in \mathbb{R}^p$, sample $(\x, y)$.
\State $\w_1 = \w_0 - \eta \evalat{\partial \rho (\w)}{\w_0}$ \Comment{Initialization $\w_0 - A \aalpha$ with $\aalpha = \bm{1}_y$}
\State $\lambda_1 = 0$ \Comment{Initialization $\bb^\top \aalpha$ with $\aalpha = \bm{1}_y$}
\State $t = 1$
\While {not converged}
    \State Choose direction $\s_t \in \PP$ \Comment{(e.g. conditional gradient or smoothed loss)} \label{algo_line:dual_direction}
    \State $\w_\s = -A \s_t$ \label{algo_line:primal_direction}
    \State $\lambda_\s = \bb^\top \s_t$
    \State $\gamma_t = \dfrac{ (\w_t - \w_0 - \w_\s)^\top (\w_t - \w_0)  + \eta (\lambda_\s - \lambda_t)}{ \|\w - \w_0 - \w_\s\|^2}$ \Comment{Optimal- step-size} \label{algo_line:step_size_pfw}
    \State $\w_{t+1} = (1 - \gamma_t) \w_t + \gamma_t (\w_\s + \w_0)$ \Comment{$A \aalpha_{t+1} = (1 - \gamma_t) A \aalpha_t + \gamma_t A\s_t$} \label{algo_line:update_pfw}
    \State $\lambda_{t+1} = (1 - \gamma_t) \lambda_t + \gamma_t \lambda_\s$ \Comment{$\bb^\top \aalpha_{t+1} = (1 - \gamma_t) \bb^\top\aalpha_t + \gamma_t \bb^\top\s_t$}
    \State $t = t+1$
\EndWhile
\end{algorithmic}
\end{algorithm}

Note that when $\f_\x$ is linear, and when the search direction $\s$ is given by the conditional gradient, we recover the standard Frank-Wolfe algorithm for SVM \citep{Lacoste-Julien2013}.

\subsection{Single-Step Proximal Frank-Wolfe Algorithm}
\label{sec:appendix_simplified_fw_algo}

We now provide some simplification to the steps \ref{algo_line:primal_direction}, \ref{algo_line:step_size_pfw} and \ref{algo_line:update_pfw} of Algorithm \ref{algo:proxfw} when a single step is taken, as is the case in the DFW algorithm. This corresponds to the iteration $t=1$.

\begin{customthm}{2}[Cost per iteration, detailed]
Suppose that a single step is performed on the dual of (\ref{eq:proxfw_step_min}).
If $\rho = 0$, its conditional gradient is exactly given by $-\evalat{\partial \LL_{y}(\f_\x(\w))}{\w_t}$.
If $\rho \neq 0$, the direction $-\evalat{\partial \left(\rho(\w) + \LL_{y}(\f_\x(\w)) \right)}{\w_t}$ is feasible and is typically a close approximation of the exact conditional gradient. 
The resulting update can be written as:
\begin{equation} \label{eq:primal_update_appendix}
\w_{t+1} = \w_t - \eta \left[\evalat{\partial \rho(\w)}{\w_t} + \gamma \evalat{\partial \LL_j(\f_j(\w))}{\w_t} \right]
\end{equation}
\end{customthm}
\begin{proof}
We begin by noting that $\evalat{\partial \left(\rho(\w) + \LL_{y}(\f_\x(\w)) \right)}{\w_t}$ always corresponds to a feasible direction in the dual, since it is equal to $A s_\star$, where $s_\star$ is the one-hot encoding of $y_\star = \argmax_{\ybar \in \Y} b_{\ybar}$.

When $\rho = 0$, we have that initially $A \alpha_0 = 0$, and thus $s_\star$ corresponds exactly to the dual conditional gradient.

When $\rho \neq 0$, the exact dual conditional gradient is given by a one-hot-encoding of $\argmax\limits_{\ybar \in \Y} \{ \mathbf{a}_{\ybar}^\top (\w_1 - \w_0) + b_{\ybar} \}$ instead of $s_\star$.
Typically, we have that $\w_1 \simeq \w_0$ because $\eta \evalat{\partial \rho (\w)}{\w_0}$ is usually small in comparison to $\w_0$.
Therefore the direction can be considered to be a close approximation.

We now prove equation (\ref{eq:primal_update_appendix}) in the next lemma.
\end{proof}

\begin{lemma}
Suppose that we apply the Proximal Frank-Wolfe algorithm with a single step.
Let $\ddelta_t = \partial  \left[\s_t^\top  (f_{\x, \ybar}(\w_0) - f_{\x, y}(\w_0))_{\ybar \in \Y }\right]$ and $\rr_t = \partial_w \rho(\w_0)$.
Then we can rewrite step \ref{algo_line:primal_direction} as:
\begin{align}
\w_s
    &= -\eta \left[\rr_t + \ddelta_t \right]. \label{app:eq:direction_simplified}
\end{align}
In addition, we can simplify steps \ref{algo_line:step_size_pfw} and \ref{algo_line:update_pfw} of Algorithm \ref{algo:proxfw} to:
\begin{align}
\gamma_t
    &= \frac{ -\eta \ddelta_t ^\top \rr_t  + \s_t^\top \bb}{ \eta\| \ddelta_t\|^2} \text{ clipped to [0, 1]}, \label{app:eq:step_size_simplified} \\
\w_{t+1}
    &= \w_0 - \eta \left[\rr_t + \gamma_t \ddelta_t \right]. \label{app:eq:step_simplified}
\end{align}
\end{lemma}

\begin{proof}
Again, since we perform a single step of FW, we assume $t=1$. To prove equation (\ref{app:eq:direction_simplified}), we note that:
\begin{equation}
\begin{split}
\w_s
    &= - A \s, \\
    &= - \eta \left[ \evalat{\partial \rho(\w)}{\w_0} + \left((\evalat{\partial f_{\x, \ybar}(\w)}{\w_0} - \evalat{\partial f_{\x, y}(\w)}{\w_0} )_{\ybar \in \Y } \right)^\top \: \s_t \right], \\
    &= - \eta \left[ \rr_t + \ddelta_t\right].
\end{split}
\end{equation}
We point out the two following results:
\begin{equation}
\w_t - \w_0 = \w_1 - \w_0 = - \eta \evalat{\partial \rho(\w)}{\w_0} = -\eta \rr_t, \label{app:eq:simplification_1}
\end{equation}
and:
\begin{equation}
\w_t - \w_0 - \w_s = -\eta \rr_t + \eta \rr_t + \eta \ddelta_t = \eta \ddelta_t. \label{app:eq:simplification_2}
\end{equation}
Since $\lambda_1=0$ by definition, equation (\ref{app:eq:step_size_simplified}) is obtained with a simple application of equations \ref{app:eq:simplification_1} and \ref{app:eq:simplification_2}.
Finally, we prove equation \ref{app:eq:step_simplified} by writing:
\begin{equation} \label{app:eq:update_simplified_proof}
\begin{split}
\w_{t+1}
    &= (1 - \gamma_t) \w_t + \gamma_t (\w_\s + \w_0), \\
    &= (1 - \gamma_t) (\w_0 - \eta \rr_t) + \gamma_t (-\eta \rr_t - \eta \ddelta_t + \w_0), \\
    &= \w_0 - \eta (\rr_t + \gamma_t  \ddelta_t). \\
\end{split}
\end{equation}

\end{proof}

\subsection{Smoothing the Loss}
\label{sec:appendix_smoothing}

As pointed out in the paper, the SVM loss is non-smooth and has sparse derivatives, which can prevent the effective training of deep neural networks \citep{Berrada2018}.
Partial linearization can solve this problem by locally smoothing the dual \citep{Mohapatra2016}.
However, this would introduce a temperature hyper-parameter which is undesirable.
Therefore, we note that DFW can be applied with any direction that is feasible in the dual, since it computes an optimal step-size.
In particular, the following result states that we can use the well-conditioned and non-sparse gradient of cross-entropy.
\begin{proposition} \label{prop:ce_direction}
The gradient of cross-entropy in the primal gives a feasible direction in the dual. Furthermore, we can inexpensively detect when this feasible direction cannot provide any improvement in the dual, and automatically switch to the conditional gradient when that is the case.
\end{proposition}

For simplicity, we divide Proposition \ref{prop:ce_direction} into two distinct parts: first we show how the CE gradient gives a feasible direction in the dual, and then how it can be detected to be an ascent direction.

\begin{lemma} \label{lemma:ce_direction_feasible}
The gradient of cross-entropy in the primal gives a feasible direction in the dual. In other words, the gradient of cross-entropy $\bm{g}$ in the primal is such that there exists a dual search direction $\s \in \PP$ verifying $\bm{g} = - A \s$.
\end{lemma}
\begin{proof}
We consider the vector of scores $\left(f_{\x, \ybar}(\w)\right)_{\ybar \in \mathcal{Y}} \in \mathbb{R}^{|\Y|}$. We compute its softmax: $\s_{\text{ce}} = \left( \frac{ \exp \left(f_{\x, \ybar}(\w) \right) }{\sum_{j \in \Y} \exp \left(f_{\x, j}(\w) \right)} \right)_{\ybar \in \Y}$. Clearly, $\s_{\text{ce}} \in \PP$ by property of the softmax. Furthermore, by going back to the definition of $A$, one can easily verify that $-A \s_{\text{ce}}$ is exactly the primal gradient given by a backward pass through the cross-entropy loss instead of the hinge loss. This concludes the proof.
\end{proof}

The previous lemma has shown that we can use the gradient of cross-entropy as a feasible direction $\s_{\text{ce}}$ in the dual.
The next step is to make it a dual ascent direction, that is a direction which always permits improvement on the dual objective (unless at the optimal point).
In what follows, we show that we can inexpensively (approximately) compute a sufficient condition for $\s_{\text{ce}}$ to be an ascent direction.
If the condition is not satisfied, then we can automatically switch to use the subgradient of the hinge loss (which is known as an ascent direction in the dual).

\begin{lemma}
Let $\s \in \PP$ be a feasible direction in the dual, and $\vv = \left( \T_{\w_0}\f_{\x}(\w_t)_{\ybar} + \Delta(\ybar, y) - \T_{\w_0}\f_{\x}(\w_t)_y \right)_{\ybar \in \Y} \in \mathbb{R}^{|\Y|}$ be the vector of augmented scores output by the linearized model.
Let us assume that we apply the single-step Proximal Frank-Wolfe algorithm (that is, we have $t=1$), and that $\rho$ is a non-negative function. \\
Then $\s^\top \vv > 0$ is a sufficient condition for $\s$ to be an ascent direction in the dual.
\end{lemma}

\begin{proof}
Let $\s \in \PP$, $\vv = \left( \T_{\w_0}\f_{\x}(\w_t)_{\ybar} + \Delta(\ybar, y) - \T_{\w_0}\f_{\x}(\w_t)_y \right)_{\ybar \in \Y}$. By definition, we have that:
\begin{equation}
\begin{split}
\vv
    &= \left(\mathbf{a}_{\ybar}^\top (\w_t - \w_0) + b_{\ybar} - \T_{\w_0}\rho(\w) \right)_{\ybar \in \Y}, \\
    &= \frac{1}{\eta} A^\top (\w_t - \w_0) + \bb -  \left(\T_{\w_0}\rho(\w) \right)_{\ybar \in \Y}.
\end{split}
\end{equation}
Therefore:
\begin{equation}
\begin{split}
    &\s^\top \vv > 0 \\
    \iff& \frac{1}{\eta} (A \s)^\top (\w_t - \w_0) + \s^\top\bb -  \s^\top\left(\T_{\w_0}\rho(\w) \right)_{\ybar \in \Y} > 0, \\
    \iff& (A \s)^\top (\w_t - \w_0) + \eta \s^\top \bb - \eta \T_{\w_0}\rho(\w)  > 0, \quad \text{(since $\s \in \PP$ and $\eta > 0$)} \\
    \iff& -\w_s ^\top (\w_t - \w_0) + \eta \s^\top \bb - \eta \rho(\w_0) - \eta \partial\rho(\w_0)^\top(\w_t - \w_0) > 0, \\
    \iff& -\w_s ^\top (\w_t - \w_0) + \eta \s^\top \bb - \eta \rho(\w_0) + (\w_t - \w_0)^\top(\w_t - \w_0) > 0, \\
    \iff& (\w_t - \w_0 -\w_s) ^\top (\w_t - \w_0) + \eta \s^\top \bb - \eta \rho(\w_0) > 0, \\
    \implies& (\w_t - \w_0 -\w_s) ^\top (\w_t - \w_0) + \eta \s^\top \bb > 0, \quad \text{(because $\rho(\w_0) \geq 0$)} \\
    \iff& \gamma_t > 0 \quad \text{(we have that $\lambda_t=0$ at $t=1$)}.
\end{split}
\end{equation}

We have just shown that if $\s^\top \vv > 0$, then $\gamma_t > 0$. Since $\gamma_t$ is an optimal step-size, this indicates that $\s$ is an ascent direction (we would obtain $\gamma_t=0$ for a direction $\s$ that cannot provide improvement).
\end{proof}

\paragraph{Approximate Condition.} In practice, we consider that $\T_{\w_0}\f_{\x}(\w_t) \simeq \f_{\x}(\w_0)$.
Indeed, for $t=1$, we have that $\| \T_{\w_0}\f_{\x}(\w) - \f_{\x}(\w_0) \| = \mathcal{O}(\|\w_t - \w_0\|)$, and $\|\w_t - \w_0\| = \|\eta \partial_w \rho(\w_0))\|$, which is typically very small (we use a weight decay coefficient in the order of $1E^{-4}$ in our experimental settings).
Therefore, we replace $\T_{\w_0}\f_{\x}(\w)$ by $\f_{\x}(\w_0)$ in the above criterion, which becomes inexpensive since $\f_{\x}(\w_0)$ is already computed by the forward pass.

\subsection{Nesterov Momentum}
\label{sec:appendix_nesterov}

As can be seen in the previous primal-dual algorithms, taking a step in the dual can be decomposed into two stages: the initialization and the movement along the search direction.
The initialization step is not informative about the optimization problem.
Therefore, we discard it from the momentum velocity, and only accumulate the step along the conditional gradient (scaled by $\gamma_t \eta$).
This results in the following velocity update:
\begin{equation}
\z_{t+1} =  \mu \z_t - \eta \gamma_t (\rr_t + \ddelta_t).
\end{equation}

\vfill
\pagebreak

\section{Experimental Details on the CIFAR Data Sets}
\label{sec:details_cifar}

\subsection{Adaptive Gradient Baselines: Cross-Validation (Without Data Augmentation)}
\label{sec:cross_validation_cifar}

\begin{table}[H]
\centering
\begin{tabular}{cccc}
\toprule
               $l_2$ &             $\eta$ & Accuracy CIFAR-10 (\%) & Accuracy CIFAR-100 (\%) \\
\midrule
           0.0001 &           0.001 &                  71.6 &                  39.44 \\
 {\bf {\bf 0.0001}} &  {\bf {\bf 0.01}} &            {\bf 88.18} &             {\bf 55.72} \\
           0.0001 &             0.1 &                  86.4 &                  55.44 \\
           0.0001 &               1 &                 68.48 &                  20.68 \\
\
\end{tabular}
\label{tab:adagrad_dn}
\caption{Cross-Validation for ADAGRAD on DN architecture         (best validation accuracy obtained during training).}
\end{table}

\begin{table}[H]
\centering
\begin{tabular}{cccc}
\toprule
          $l_2$ &        $\eta$ & Accuracy CIFAR-10 (\%) & Accuracy CIFAR-100 (\%) \\
\midrule
      0.0001 &      0.001 &                 68.98 &                  31.86 \\
 {\bf 0.0001} &  {\bf 0.01} &             {\bf 86.4} &                  53.82 \\
      0.0001 &        0.1 &                  83.6 &                  51.18 \\
      0.0005 &      0.001 &                 68.66 &                   32.5 \\
 {\bf 0.0005} &  {\bf 0.01} &                  86.3 &             {\bf 56.16} \\
      0.0005 &        0.1 &                 77.92 &                  44.12 \\
\
\end{tabular}
\label{tab:adagrad_wrn}
\caption{Cross-Validation for ADAGRAD on WRN architecture         (best validation accuracy obtained during training).}
\end{table}

\begin{table}[H]
\centering
\begin{tabular}{cccc}
\toprule
               $l_2$ &              $\eta$ & Accuracy CIFAR-10 (\%) & Accuracy CIFAR-100 (\%) \\
\midrule
           0.0001 &           0.0001 &                 86.26 &                   50.7 \\
 {\bf {\bf 0.0001}} &  {\bf {\bf 0.001}} &            {\bf 89.42} &              {\bf 63.9} \\
           0.0001 &             0.01 &                 81.12 &                  51.82 \\
\
\end{tabular}
\label{tab:adam_dn}
\caption{Cross-Validation for ADAM on DN architecture         (best validation accuracy obtained during training).}
\end{table}

\begin{table}[H]
\centering
\begin{tabular}{cccc}
\toprule
               $l_2$ &              $\eta$ & Accuracy CIFAR-10 (\%) & Accuracy CIFAR-100 (\%) \\
\midrule
           0.0001 &           0.0001 &                  79.7 &                  41.42 \\
 {\bf {\bf 0.0001}} &  {\bf {\bf 0.001}} &             {\bf 86.1} &              {\bf 58.7} \\
           0.0001 &             0.01 &                 80.06 &                  50.86 \\
           0.0005 &           0.0001 &                 78.88 &                  40.08 \\
           0.0005 &            0.001 &                 85.14 &                  55.26 \\
           0.0005 &             0.01 &                 72.54 &                  36.82 \\
\
\end{tabular}
\label{tab:adam_wrn}
\caption{Cross-Validation for ADAM on WRN architecture         (best validation accuracy obtained during training).}
\end{table}

\begin{table}[H]
\centering
\begin{tabular}{cccc}
\toprule
               $l_2$ &              $\eta$ & Accuracy CIFAR-10 (\%) & Accuracy CIFAR-100 (\%) \\
\midrule
           0.0001 &           0.0001 &                 84.28 &                  49.54 \\
 {\bf {\bf 0.0001}} &  {\bf {\bf 0.001}} &             {\bf 90.4} &             {\bf 68.54} \\
           0.0001 &             0.01 &                 83.98 &                  50.44 \\
\
\end{tabular}
\label{tab:amsgrad_dn}
\caption{Cross-Validation for AMSGRAD on DN architecture         (best validation accuracy obtained during training).}
\end{table}

\begin{table}[H]
\centering
\begin{tabular}{cccc}
\toprule
               $l_2$ &              $\eta$ & Accuracy CIFAR-10 (\%) & Accuracy CIFAR-100 (\%) \\
\midrule
           0.0001 &           0.0001 &                 75.86 &                   41.6 \\
 {\bf {\bf 0.0001}} &  {\bf {\bf 0.001}} &            {\bf 87.02} &              {\bf 59.6} \\
           0.0001 &             0.01 &                 82.32 &                  52.12 \\
           0.0005 &           0.0001 &                 75.74 &                  42.28 \\
           0.0005 &            0.001 &                 86.16 &                  57.82 \\
           0.0005 &             0.01 &                 75.82 &                  36.48 \\
\
\end{tabular}
\label{tab:amsgrad_wrn}
\caption{Cross-Validation for AMSGRAD on WRN architecture         (best validation accuracy obtained during training).}
\end{table}

\begin{table}[H]
\centering
\begin{tabular}{cccc}
\toprule
               $l_2$ &            $\eta$ & Accuracy CIFAR-10 (\%) & Accuracy CIFAR-100 (\%) \\
\midrule
           0.0001 &          0.001 &                 72.72 &                  40.96 \\
           0.0001 &           0.01 &                 83.26 &                  53.12 \\
 {\bf {\bf 0.0001}} &  {\bf {\bf 0.1}} &             {\bf 91.7} &              {\bf 59.7} \\
           0.0001 &              1 &                 10.16 &                   1.16 \\
\
\end{tabular}
\label{tab:bpgrad_dn}
\caption{Cross-Validation for BPGRAD on DN architecture         (best validation accuracy obtained during training).}
\end{table}

\begin{table}[H]
\centering
\begin{tabular}{cccc}
\toprule
          $l_2$ &        $\eta$ & Accuracy CIFAR-10 (\%) & Accuracy CIFAR-100 (\%) \\
\midrule
      0.0001 &      0.001 &                 64.98 &                   31.9 \\
      0.0001 &       0.01 &                 78.46 &                  44.26 \\
 {\bf 0.0001} &   {\bf 0.1} &            {\bf 89.24} &                  54.42 \\
      0.0001 &          1 &                  16.1 &                   1.16 \\
      0.0005 &      0.001 &                 68.08 &                  33.26 \\
 {\bf 0.0005} &  {\bf 0.01} &                 85.44 &              {\bf 59.9} \\
      0.0005 &        0.1 &                 88.44 &                  51.28 \\
      0.0005 &          1 &                 10.16 &                   1.16 \\
\
\end{tabular}
\label{tab:bpgrad_wrn}
\caption{Cross-Validation for BPGRAD on WRN architecture         (best validation accuracy obtained during training).}
\end{table}

\subsection{Adaptive Gradient Baselines: Cross-Validation (With Data Augmentation)}
\label{sec:cross_validation_cifar_augmented}

\begin{table}[H]
    \centering
    \begin{tabular}{ccccc}
    \toprule
              $l_2$ &         $\eta$ &   batchsize & Accuracy CIFAR-10 (\%) & Accuracy CIFAR-100 (\%) \\
    \midrule
          0.0001 &      0.0001 &          64 &                 90.38 &                   61.6 \\
          0.0001 &      0.0001 &         128 &                 87.86 &                  57.82 \\
          0.0001 &      0.0001 &         256 &                 86.66 &                  53.64 \\
     {\bf 0.0001} &  {\bf 0.001} &     {\bf 64} &                 92.52 &             {\bf 69.66} \\
     {\bf 0.0001} &  {\bf 0.001} &  {\bf 128} &            {\bf 92.72} &                   69.5 \\
          0.0001 &       0.001 &         256 &                 92.64 &                  67.56 \\
          0.0001 &        0.01 &          64 &                  82.1 &                     45 \\
          0.0001 &        0.01 &         128 &                  83.9 &                   53.4 \\
          0.0001 &        0.01 &         256 &                 86.86 &                   58.1 \\
    \
    \end{tabular}
    \label{tab:amsgrad_dn_augmented}
    \caption{Cross-Validation for AMSGRAD on DN architecture with data augmentation (best validation accuracy obtained during training).}
    \end{table}

    \begin{table}[H]
    \centering
    \begin{tabular}{ccccc}
    \toprule
              $l_2$ &         $\eta$ &   batchsize & Accuracy CIFAR-10 (\%) & Accuracy CIFAR-100 (\%) \\
    \midrule
          0.0001 &      0.0001 &         128 &                  90.5 &                  64.36 \\
          0.0001 &      0.0001 &         256 &                  89.6 &                  62.02 \\
          0.0001 &      0.0001 &         512 &                 88.26 &                  58.68 \\
     {\bf 0.0001} &  {\bf 0.001} &    {\bf 128} &                  91.7 &              {\bf 69.3} \\
          0.0001 &       0.001 &         256 &                  91.8 &                  68.98 \\
     {\bf 0.0001} &  {\bf 0.001} &  {\bf 512} &            {\bf 91.88} &                  68.64 \\
          0.0001 &        0.01 &         128 &                 83.36 &                  53.72 \\
          0.0001 &        0.01 &         256 &                 84.58 &                  57.28 \\
          0.0001 &        0.01 &         512 &                 87.42 &                  60.68 \\
          0.0005 &      0.0001 &         128 &                 91.44 &                  65.52 \\
          0.0005 &      0.0001 &         256 &                  89.7 &                  61.98 \\
          0.0005 &      0.0001 &         512 &                 88.48 &                   59.1 \\
          0.0005 &       0.001 &         128 &                 90.82 &                  67.38 \\
          0.0005 &       0.001 &         256 &                    91 &                  67.58 \\
          0.0005 &       0.001 &         512 &                 91.06 &                  67.06 \\
          0.0005 &        0.01 &         128 &                  72.6 &                   34.8 \\
          0.0005 &        0.01 &         256 &                 76.56 &                  41.82 \\
          0.0005 &        0.01 &         512 &                 79.12 &                   45.6 \\
    \
    \end{tabular}
    \label{tab:amsgrad_wrn_augmented}
    \caption{Cross-Validation for AMSGRAD on WRN architecture with data augmentation (best validation accuracy obtained during training).}
    \end{table}

    \begin{table}[H]
    \centering
    \begin{tabular}{ccccc}
    \toprule
              $l_2$ &       $\eta$ &  batchsize & Accuracy CIFAR-10 (\%) & Accuracy CIFAR-100 (\%) \\
    \midrule
          0.0001 &      0.01 &         64 &                  92.8 &                  69.12 \\
          0.0001 &      0.01 &        128 &                 91.38 &                  66.26 \\
          0.0001 &      0.01 &        256 &                 89.46 &                  60.68 \\
     {\bf 0.0001} &  {\bf 0.1} &  {\bf 64} &            {\bf 95.34} &                  68.04 \\
          0.0001 &       0.1 &         64 &                 95.34 &                  66.78 \\
          0.0001 &       0.1 &        128 &                  94.7 &                  73.62 \\
     {\bf 0.0001} &  {\bf 0.1} &   {\bf 128} &                  94.7 &              {\bf 74.1} \\
          0.0001 &       0.1 &        256 &                    94 &                   70.9 \\
          0.0001 &         1 &         64 &                 77.04 &                  38.44 \\
          0.0001 &         1 &        128 &                 82.56 &                  52.12 \\
          0.0001 &         1 &        256 &                 87.38 &                   60.6 \\
          0.0001 &        10 &         64 &                 10.16 &                   1.16 \\
          0.0001 &        10 &        128 &                 10.16 &                   1.16 \\
          0.0001 &        10 &        256 &                 10.56 &                   1.62 \\
    \
    \end{tabular}
    \label{tab:dfw_dn_augmented}
    \caption{Cross-Validation for DFW on DN architecture with data augmentation (best validation accuracy obtained during training).}
    \end{table}

    \begin{table}[H]
    \centering
    \begin{tabular}{ccccc}
    \toprule
              $l_2$ &       $\eta$ &   batchsize & Accuracy CIFAR-10 (\%) & Accuracy CIFAR-100 (\%) \\
    \midrule
          0.0001 &      0.01 &         128 &                 93.24 &                  71.18 \\
          0.0001 &      0.01 &         256 &                  91.8 &                  67.36 \\
          0.0001 &      0.01 &         512 &                  90.9 &                  64.74 \\
          0.0001 &       0.1 &         128 &                 94.18 &                  74.26 \\
          0.0001 &       0.1 &         256 &                 94.66 &                  73.24 \\
          0.0001 &       0.1 &         512 &                 94.02 &                  71.66 \\
          0.0001 &         1 &         128 &                  84.7 &                   55.1 \\
          0.0001 &         1 &         256 &                 89.62 &                  61.88 \\
     {\bf 0.0001} &  {\bf 1} &  {\bf 512} &            {\bf 94.98} &                  73.94 \\
          0.0001 &        10 &         128 &                 10.72 &                   1.32 \\
          0.0001 &        10 &         256 &                 10.72 &                   4.96 \\
          0.0001 &        10 &         512 &                 13.62 &                    6.4 \\
          0.0005 &      0.01 &         128 &                  94.1 &                  72.14 \\
          0.0005 &      0.01 &         256 &                 92.72 &                  69.96 \\
          0.0005 &      0.01 &         512 &                 90.84 &                  64.04 \\
          0.0005 &       0.1 &         128 &                 88.86 &                  63.06 \\
     {\bf 0.0005} &  {\bf 0.1} &    {\bf 256} &                 94.68 &             {\bf 75.34} \\
          0.0005 &       0.1 &         512 &                  94.1 &                  72.54 \\
          0.0005 &         1 &         128 &                  63.3 &                   26.9 \\
          0.0005 &         1 &         256 &                 72.08 &                  38.28 \\
          0.0005 &         1 &         512 &                 80.74 &                   48.1 \\
          0.0005 &         1 &         512 &                 80.74 &                  44.52 \\
          0.0005 &        10 &         128 &                 10.72 &                   1.36 \\
          0.0005 &        10 &         256 &                 10.72 &                    1.3 \\
          0.0005 &        10 &         512 &                 14.18 &                   1.98 \\
    \
    \end{tabular}
    \label{tab:dfw_wrn_augmented}
    \caption{Cross-Validation for DFW on WRN architecture with data augmentation (best validation accuracy obtained during training).}
    \end{table}

\subsection{Convergence Plots}

In this section we provide the convergence plots of the different algorithms on the CIFAR data sets without data augmentation.
In some cases the training performance can show some oscillations.
We emphasize that this is the result of cross-validating the initial learning rate based on the validation set performance: sometimes a better-behaved convergence would be obtained on the training set with a lower learning rate.
However this lower learning rate is not selected because it does not provide the best validation performance.

\begin{figure}[H]
\begin{minipage}{.47\textwidth}
\centering
\includegraphics[width=0.9\linewidth]{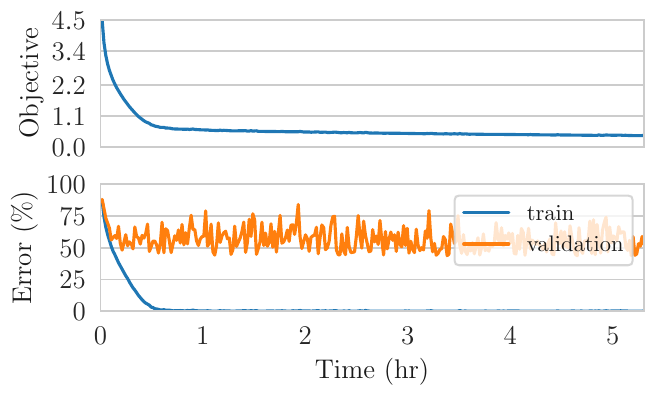}
\captionof{figure}{\em Convergence plot of Adagrad on CIFAR 100 with DN architecture.}
\label{fig:sensitivity_cifar100-dn-adagrad}
\end{minipage}%
\hfill
\begin{minipage}{.47\textwidth}
\centering
\includegraphics[width=0.9\linewidth]{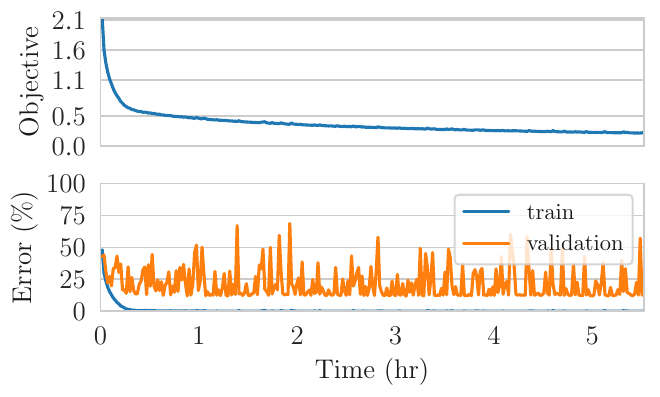}
\captionof{figure}{\em Convergence plot of Adagrad on CIFAR 10 with DN architecture.}
\label{fig:sensitivity_cifar10-dn-adagrad}
\end{minipage}
\end{figure}

\begin{figure}[H]
\begin{minipage}{.47\textwidth}
\centering
\includegraphics[width=0.9\linewidth]{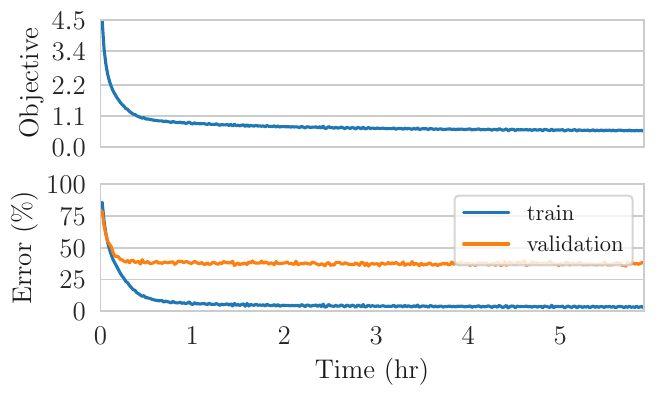}
\captionof{figure}{\em Convergence plot of Adam on CIFAR 100 with DN architecture.}
\label{fig:sensitivity_cifar100-dn-adam}
\end{minipage}%
\hfill
\begin{minipage}{.47\textwidth}
\centering
\includegraphics[width=0.9\linewidth]{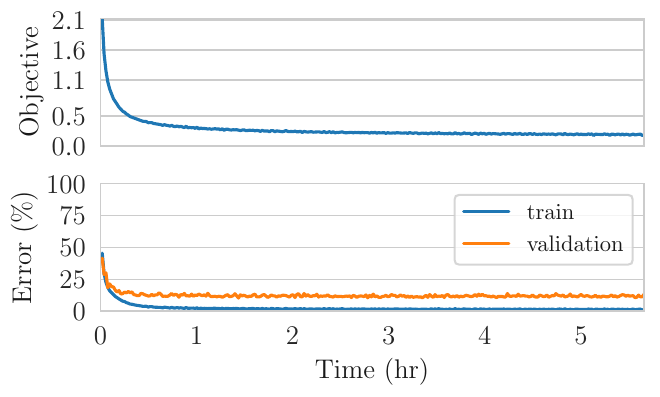}
\captionof{figure}{\em Convergence plot of Adam on CIFAR 10 with DN architecture.}
\label{fig:sensitivity_cifar10-dn-adam}
\end{minipage}
\end{figure}

\begin{figure}[H]
\begin{minipage}{.47\textwidth}
\centering
\includegraphics[width=0.9\linewidth]{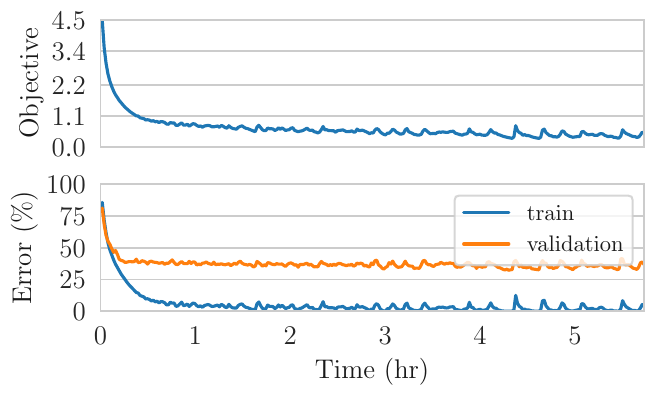}
\captionof{figure}{\em Convergence plot of AMSGrad on CIFAR 100 with DN architecture.}
\label{fig:sensitivity_cifar100-dn-amsgrad}
\end{minipage}%
\hfill
\begin{minipage}{.47\textwidth}
\centering
\includegraphics[width=0.9\linewidth]{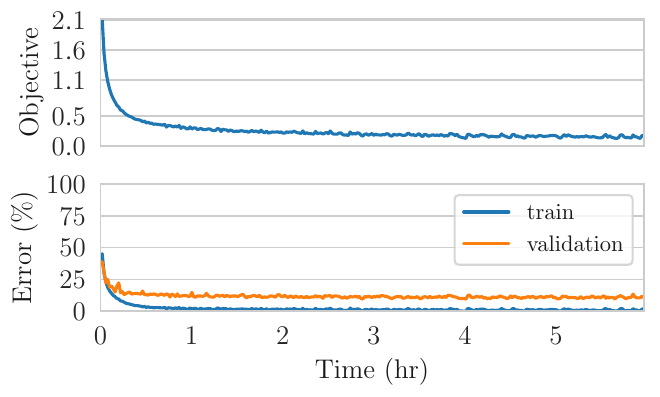}
\captionof{figure}{\em Convergence plot of AMSGrad on CIFAR 10 with DN architecture.}
\label{fig:sensitivity_cifar10-dn-amsgrad}
\end{minipage}
\end{figure}

\begin{figure}[H]
\begin{minipage}{.47\textwidth}
\centering
\includegraphics[width=0.9\linewidth]{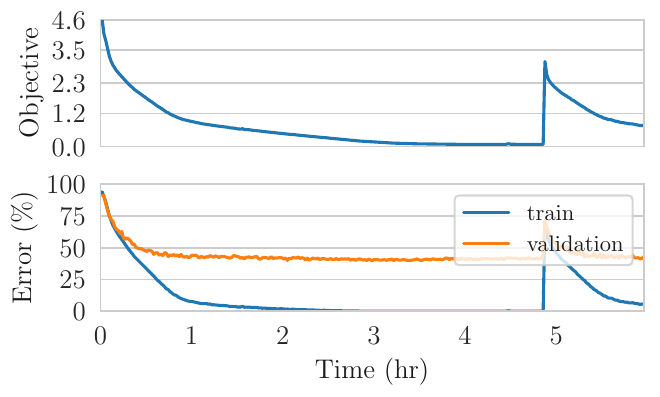}
\captionof{figure}{\em Convergence plot of BPGrad on CIFAR 100 with DN architecture.}
\label{fig:sensitivity_cifar100-dn-bpgrad}
\end{minipage}%
\hfill
\begin{minipage}{.47\textwidth}
\centering
\includegraphics[width=0.9\linewidth]{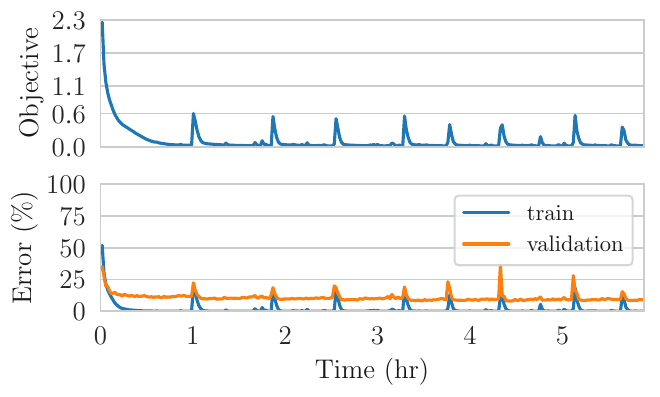}
\captionof{figure}{\em Convergence plot of BPGrad on CIFAR 10 with DN architecture.}
\label{fig:sensitivity_cifar10-dn-bpgrad}
\end{minipage}
\end{figure}

\begin{figure}[H]
\begin{minipage}{.47\textwidth}
\centering
\includegraphics[width=0.9\linewidth]{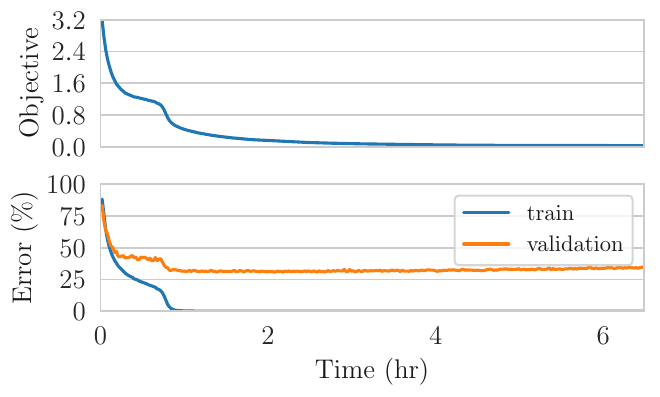}
\captionof{figure}{\em Convergence plot of DFW on CIFAR 100 with DN architecture.}
\label{fig:sensitivity_cifar100-dn-dfw}
\end{minipage}%
\hfill
\begin{minipage}{.47\textwidth}
\centering
\includegraphics[width=0.9\linewidth]{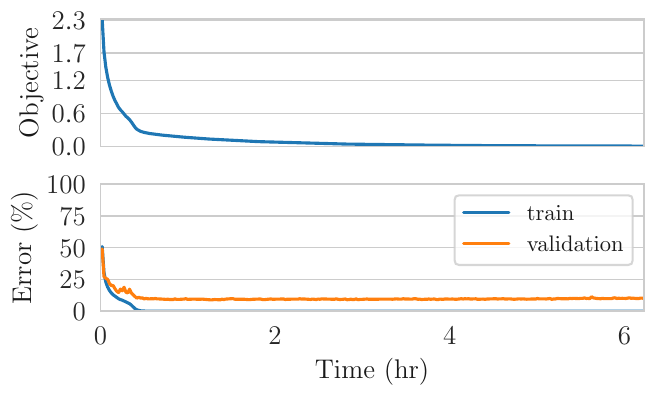}
\captionof{figure}{\em Convergence plot of DFW on CIFAR 10 with DN architecture.}
\label{fig:sensitivity_cifar10-dn-dfw}
\end{minipage}
\end{figure}

\begin{figure}[H]
\begin{minipage}{.47\textwidth}
\centering
\includegraphics[width=0.9\linewidth]{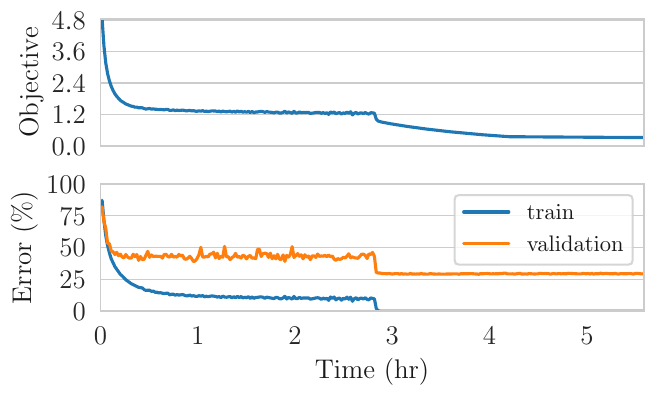}
\captionof{figure}{\em Convergence plot of SGD on CIFAR 100 with DN architecture.}
\label{fig:sensitivity_cifar100-dn-sgd}
\end{minipage}%
\hfill
\begin{minipage}{.47\textwidth}
\centering
\includegraphics[width=0.9\linewidth]{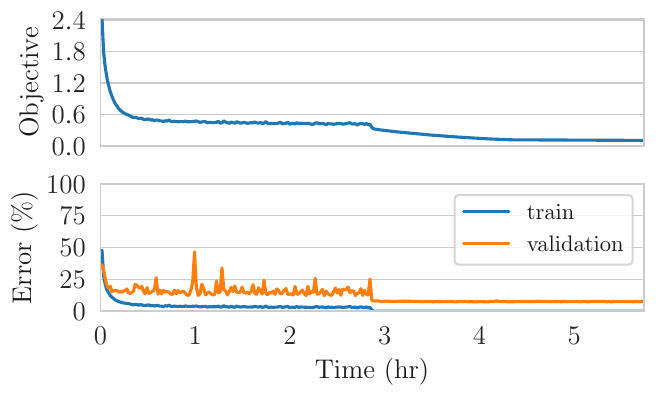}
\captionof{figure}{\em Convergence plot of SGD on CIFAR 10 with DN architecture.}
\label{fig:sensitivity_cifar10-dn-sgd}
\end{minipage}
\end{figure}

\begin{figure}[H]
\begin{minipage}{.47\textwidth}
\centering
\includegraphics[width=0.9\linewidth]{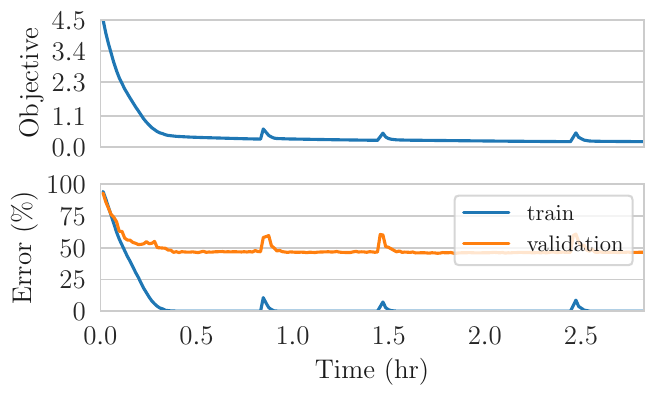}
\captionof{figure}{\em Convergence plot of Adagrad on CIFAR 100 with WRN architecture.}
\label{fig:sensitivity_cifar100-wrn-adagrad}
\end{minipage}%
\hfill
\begin{minipage}{.47\textwidth}
\centering
\includegraphics[width=0.9\linewidth]{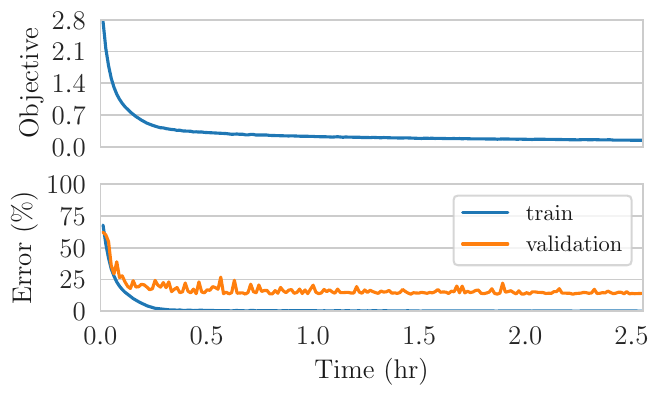}
\captionof{figure}{\em Convergence plot of Adagrad on CIFAR 10 with WRN architecture.}
\label{fig:sensitivity_cifar10-wrn-adagrad}
\end{minipage}
\end{figure}

\begin{figure}[H]
\begin{minipage}{.47\textwidth}
\centering
\includegraphics[width=0.9\linewidth]{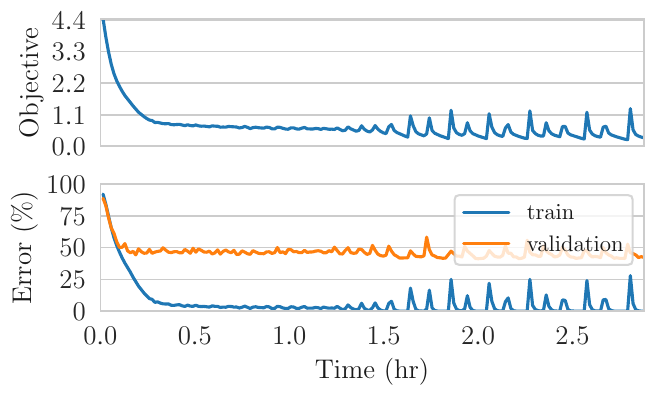}
\captionof{figure}{\em Convergence plot of Adam on CIFAR 100 with WRN architecture.}
\label{fig:sensitivity_cifar100-wrn-adam}
\end{minipage}%
\hfill
\begin{minipage}{.47\textwidth}
\centering
\includegraphics[width=0.9\linewidth]{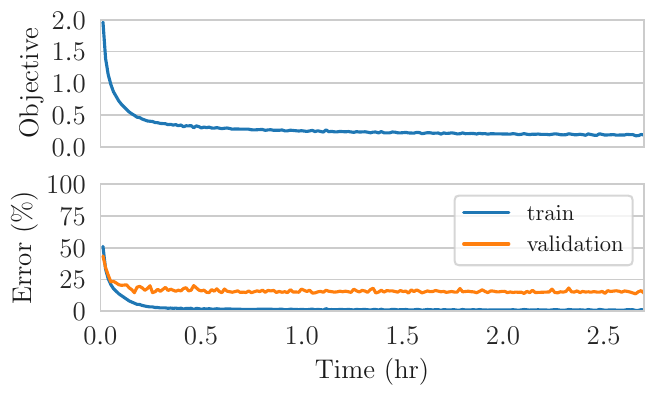}
\captionof{figure}{\em Convergence plot of Adam on CIFAR 10 with WRN architecture.}
\label{fig:sensitivity_cifar10-wrn-adam}
\end{minipage}
\end{figure}

\begin{figure}[H]
\begin{minipage}{.47\textwidth}
\centering
\includegraphics[width=0.9\linewidth]{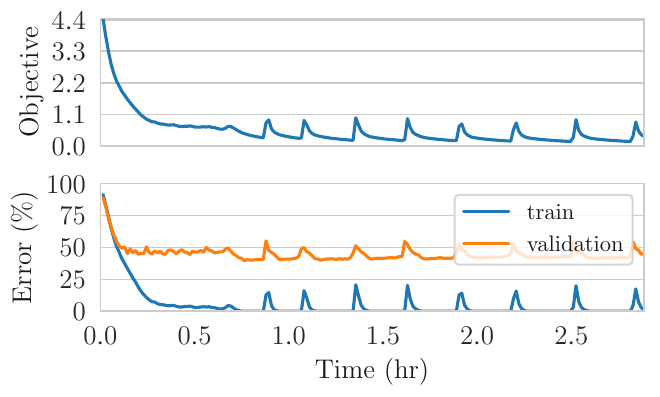}
\captionof{figure}{\em Convergence plot of AMSGrad on CIFAR 100 with WRN architecture.}
\label{fig:sensitivity_cifar100-wrn-amsgrad}
\end{minipage}%
\hfill
\begin{minipage}{.47\textwidth}
\centering
\includegraphics[width=0.9\linewidth]{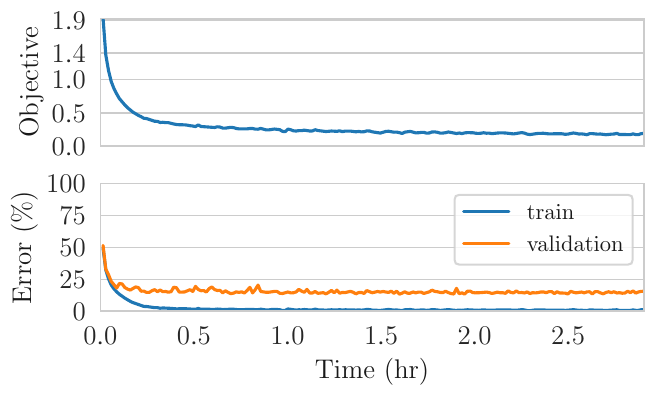}
\captionof{figure}{\em Convergence plot of AMSGrad on CIFAR 10 with WRN architecture.}
\label{fig:sensitivity_cifar10-wrn-amsgrad}
\end{minipage}
\end{figure}

\begin{figure}[H]
\begin{minipage}{.47\textwidth}
\centering
\includegraphics[width=0.9\linewidth]{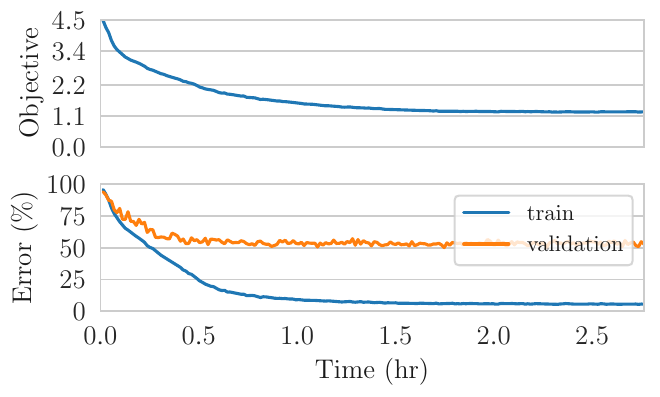}
\captionof{figure}{\em Convergence plot of BPGrad on CIFAR 100 with WRN architecture.}
\label{fig:sensitivity_cifar100-wrn-bpgrad}
\end{minipage}%
\hfill
\begin{minipage}{.47\textwidth}
\centering
\includegraphics[width=0.9\linewidth]{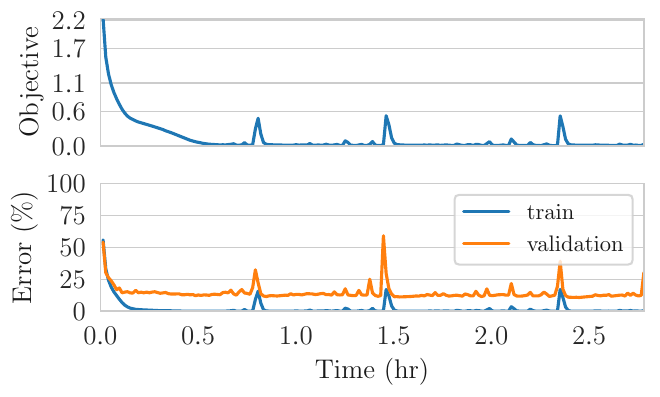}
\captionof{figure}{\em Convergence plot of BPGrad on CIFAR 10 with WRN architecture.}
\label{fig:sensitivity_cifar10-wrn-bpgrad}
\end{minipage}
\end{figure}

\begin{figure}[H]
\begin{minipage}{.47\textwidth}
\centering
\includegraphics[width=0.9\linewidth]{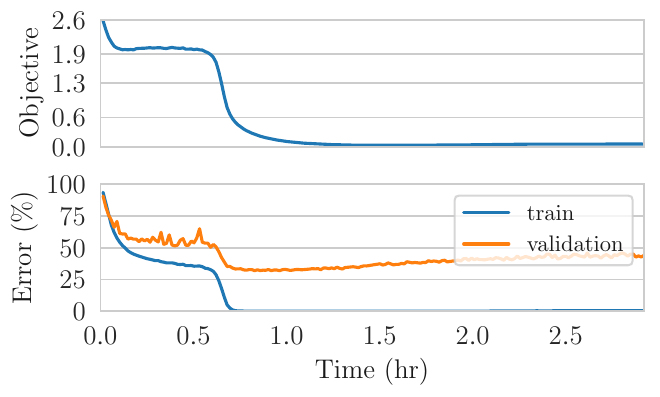}
\captionof{figure}{\em Convergence plot of DFW on CIFAR 100 with WRN architecture.}
\label{fig:sensitivity_cifar100-wrn-dfw}
\end{minipage}%
\hfill
\begin{minipage}{.47\textwidth}
\centering
\includegraphics[width=0.9\linewidth]{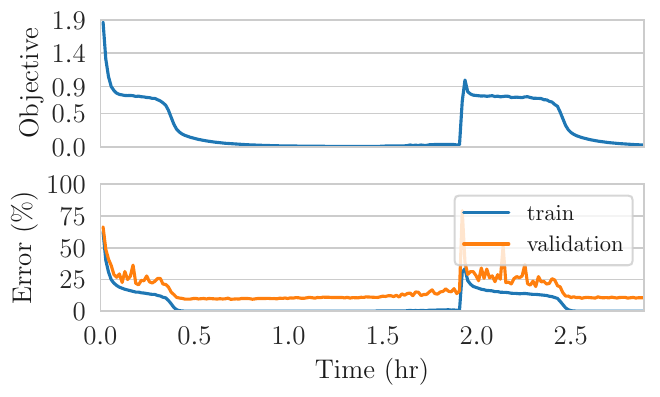}
\captionof{figure}{\em Convergence plot of DFW on CIFAR 10 with WRN architecture.}
\label{fig:sensitivity_cifar10-wrn-dfw}
\end{minipage}
\end{figure}

\begin{figure}[H]
\begin{minipage}{.47\textwidth}
\centering
\includegraphics[width=0.9\linewidth]{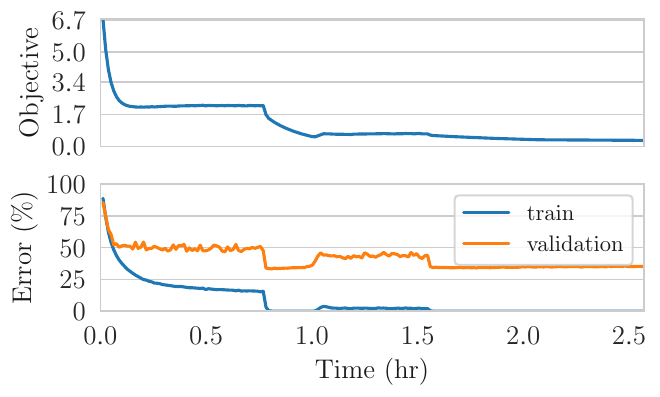}
\captionof{figure}{\em Convergence plot of SGD on CIFAR 100 with WRN architecture.}
\label{fig:sensitivity_cifar100-wrn-sgd}
\end{minipage}%
\hfill
\begin{minipage}{.47\textwidth}
\centering
\includegraphics[width=0.9\linewidth]{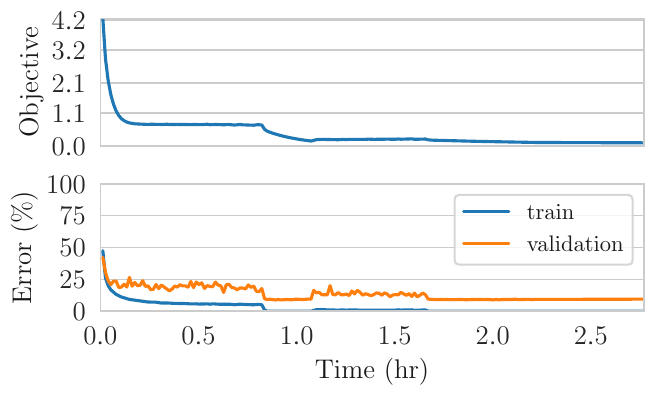}
\captionof{figure}{\em Convergence plot of SGD on CIFAR 10 with WRN architecture.}
\label{fig:sensitivity_cifar10-wrn-sgd}
\end{minipage}
\end{figure}

\subsection{SGD \& DFW: Sensitivity Analysis}
\label{sec:sensitivity_cifar}

We provide here a sensitivity analysis of the DFW algorithm on its hyper-parameter $\eta$, and we compare it against the SGD algorithm with its custom schedule.
These experiments do not use data augmentation.

\begin{figure}[H]
\centering
\includegraphics[width=\linewidth]{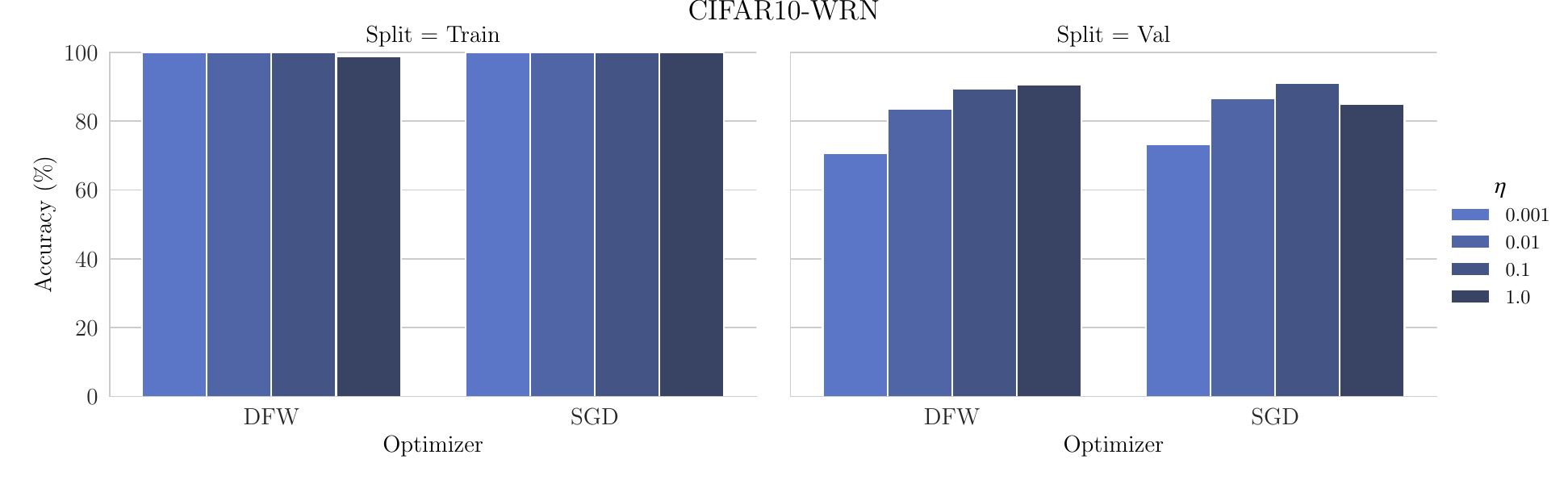}
\caption{\em Sensitivity analysis on the WRN architecture and CIFAR-10 data set.}
\label{fig:sensitivity_wrn_10}
\end{figure}

\begin{figure}[H]
\centering
\includegraphics[width=\linewidth]{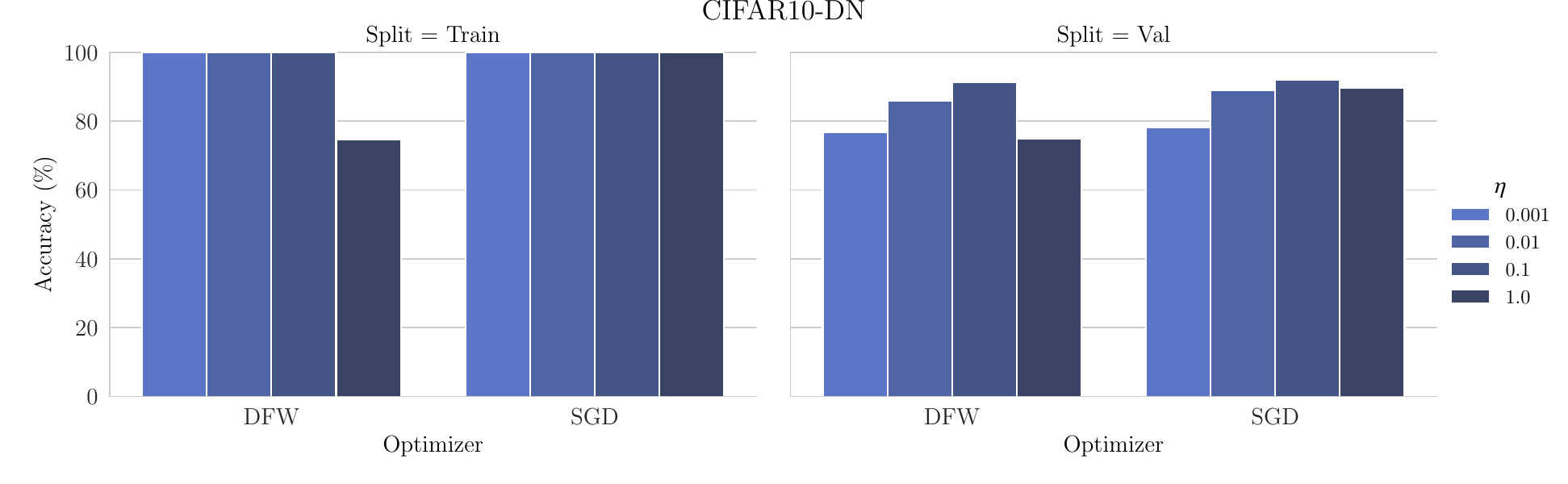}
\caption{\em Sensitivity analysis on the DN architecture and CIFAR-10 data set.}
\label{fig:sensitivity_dn_10}
\end{figure}

\begin{figure}[H]
\centering
\includegraphics[width=\linewidth]{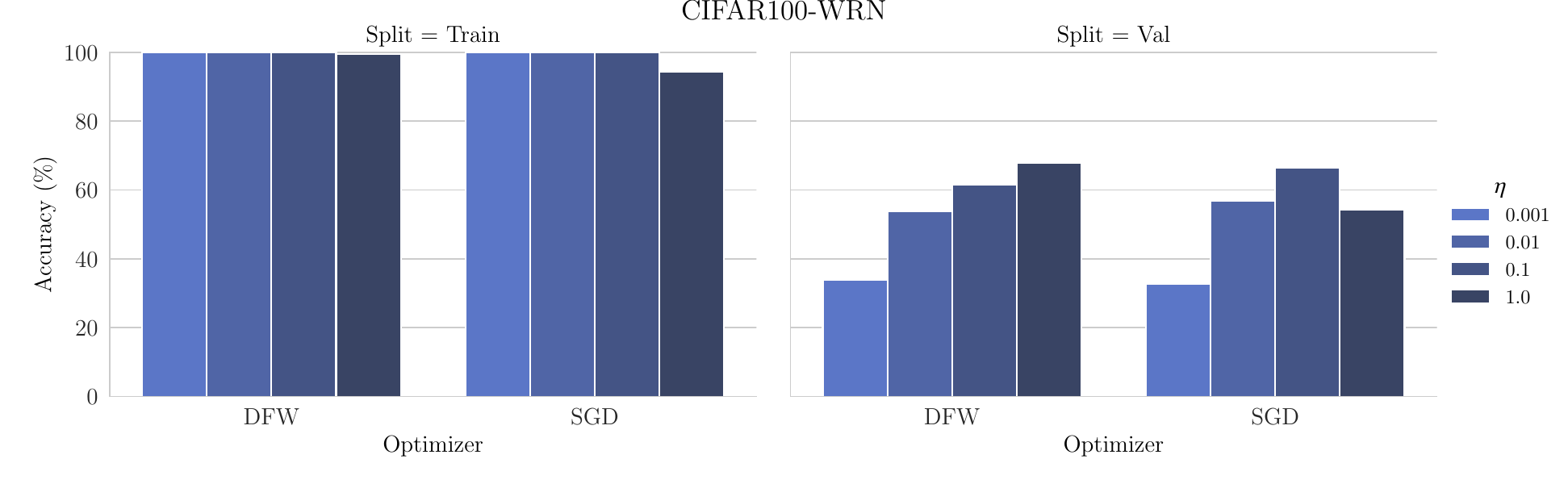}
\caption{\em Sensitivity analysis on the WRN architecture and CIFAR-100 data set.}
\label{fig:sensitivity_wrn_100}
\end{figure}

\begin{figure}[H]
\centering
\includegraphics[width=\linewidth]{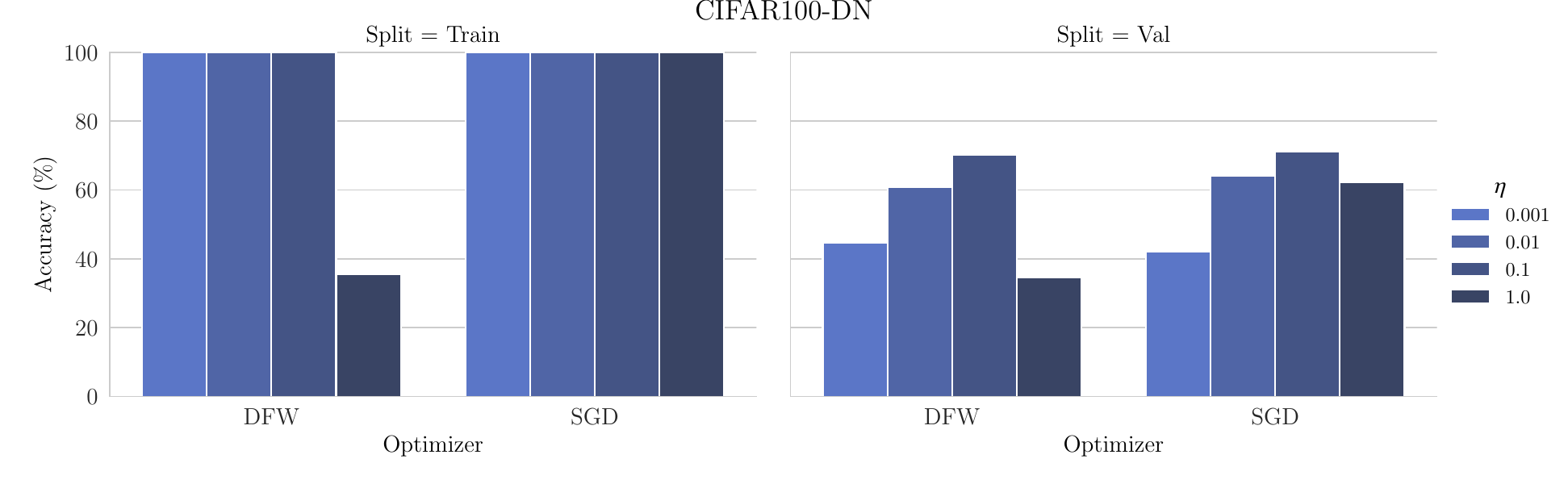}
\caption{\em Sensitivity analysis on the DN architecture and CIFAR-100 data set.}
\label{fig:sensitivity_dn_100}
\end{figure}

\vfill
\pagebreak

\section{Experimental Details on the SNLI Data Set}

\subsection{Cross-Validation}

\begin{table}[H]
\centering
\begin{tabular}{ccc}
\toprule
            $\eta$ & Accuracy CE (\%) & Accuracy SVM (\%) \\
\midrule
          0.001 &           83.43 &            84.16 \\
 {\bf {\bf 0.01}} &      {\bf 83.77} &       {\bf 84.62} \\
            0.1 &           62.09 &             34.5 \\
\
\end{tabular}
\label{tab:adagrad_blstm}
\caption{Cross-Validation for ADAGRAD on BLSTM architecture         (best validation accuracy obtained during training).}
\end{table}

\begin{table}[H]
\centering
\begin{tabular}{ccc}
\toprule
              $\eta$ & Accuracy CE (\%) & Accuracy SVM (\%) \\
\midrule
            1e-05 &           83.18 &            83.02 \\
 {\bf {\bf 0.0001}} &      {\bf 84.56} &       {\bf 84.69} \\
            0.001 &           84.42 &            83.31 \\
             0.01 &           33.82 &            33.82 \\
\
\end{tabular}
\label{tab:adam_blstm}
\caption{Cross-Validation for ADAM on BLSTM architecture         (best validation accuracy obtained during training).}
\end{table}

\begin{table}[H]
\centering
\begin{tabular}{ccc}
\toprule
              $\eta$ & Accuracy CE (\%) & Accuracy SVM (\%) \\
\midrule
1e-05 &           82.81 &            82.95 \\
 {\bf {\bf 0.0001}} &      {\bf 84.69} &       {\bf 84.83} \\
            0.001 &           84.66 &            83.59 \\
             0.01 &           36.78 &            38.25 \\
\
\end{tabular}
\label{tab:amsgrad_blstm}
\caption{Cross-Validation for AMSGRAD on BLSTM architecture         (best validation accuracy obtained during training).}
\end{table}

\begin{table}[H]
\centering
\begin{tabular}{ccc}
\toprule
           $\eta$ & Accuracy CE (\%) & Accuracy SVM (\%) \\
\midrule
         0.001 &           75.51 &            74.87 \\
          0.01 &           83.09 &            83.02 \\
           0.1 &           83.93 &            84.24 \\
 {\bf {\bf 1.0}} &      {\bf 84.28} &       {\bf 84.73} \\
            10 &           33.82 &            33.31 \\
\
\end{tabular}
\label{tab:bpgrad_blstm}
\caption{Cross-Validation for BPGRAD on BLSTM architecture         (best validation accuracy obtained during training).}
\end{table}

\begin{table}[H]
\centering
\begin{tabular}{cc}
\toprule
      $\eta$ & Accuracy (\%) \\
\midrule
      0.1 &        84.87 \\
 {\bf 1.0} &   {\bf 85.21} \\
       10 &        84.76 \\
\
\end{tabular}
\label{tab:dfw_blstm}
\caption{Cross-Validation for DFW on BLSTM architecture         (best validation accuracy obtained during training).}
\end{table}

\begin{table}[H]
\centering
\begin{tabular}{ccc}
\toprule
      $\eta$ & Accuracy CE (\%) & Accuracy SVM (\%) \\
\midrule
     0.01 &           84.22 &            84.59 \\
 {\bf 0.1} &           84.63 &       {\bf 85.15} \\
 {\bf 1.0} &      {\bf 85.06} &             84.7 \\
       10 &           34.59 &            34.51 \\
\
\end{tabular}
\label{tab:sgd_blstm}
\caption{Cross-Validation for SGD on BLSTM architecture         (best validation accuracy obtained during training).}
\end{table}

\vfill

\end{document}